\documentclass[twoside]{article}
\usepackage{Modifiedaistats2019}

\usepackage[square]{natbib}
\setcitestyle{numbers}

\usepackage{algorithmicx} 
\usepackage{algorithm}
\usepackage{algpseudocode}
\usepackage{mathrsfs}
\usepackage{amsmath} 
\usepackage{amssymb}
\usepackage{amsthm}

\usepackage{graphicx, caption, subcaption}
\usepackage{wrapfig}

\newtheorem{theorem}{Theorem}
\newtheorem{lemma}[theorem]{Lemma}

\newtheorem{corollary}[theorem]{Corollary}
\newtheorem{problem}{Problem}
\newtheorem{claim}{Claim}

\usepackage{multirow}

\usepackage{dsfont}

\title{Exploring $k$ out of Top $\rho$ Fraction of Arms in Stochastic Bandits}
\author{Wenbo Ren$^1$, Jia Liu$^2$, Ness Shroff$^1$\\
	1: Department of Computer Science, Ohio State University, Columbus, OH 43210\\
	2: Department of Computer Science, Iowa State University, Ames, Iowa 50011\\
	ren.453@osu.edu, jialiu@iastate.edu, shroff.11@osu.edu}

\begin{document}
	\twocolumn[

\aistatstitle{Exploring $k$ out of Top $\rho$ Fraction of Arms in Stochastic Bandits}

\aistatsauthor{ Wenbo Ren \And Jia Liu \And Ness B. Shroff}
\aistatsaddress{
	Dept. Computer Science \& Engineering\\
	The Ohio State University\\
	ren.453@osu.edu
	\And Dept. Computer Science\\
	Iowa State University\\
	jialiu@iastate.edu 
	\And 
	Dept. Computer Science \& Engineering\\
	The Ohio State University\\
	shroff.11@osu.edu}
]

\begin{abstract}
	This paper studies the problem of identifying any $k$ distinct arms among the top $\rho$ fraction (e.g., top 5\%) of arms from a finite or infinite set with a probably approximately correct (PAC) tolerance $\epsilon$. We consider two cases: (i) when the threshold of the top arms' expected rewards is known and (ii) when it is unknown. We prove lower bounds for the four variants (finite or infinite arms, and known or unknown threshold), and propose algorithms for each. Two of these algorithms are shown to be sample complexity optimal (up to constant factors) and the other two are optimal up to a log factor. Results in this paper provide up to $\rho n/k$ reductions compared with the ``$k$-exploration'' algorithms that focus on finding the (PAC) best $k$ arms out of $n$ arms. We also numerically show improvements over the state-of-the-art.
\end{abstract}

\section{INTRODUCTION}
\textbf{Background.} Multi-armed bandit (MAB) problems \citep{EarlyBandit1985} have been studied for decades, and well abstract the problems of decision making with uncertainty. It has been widely applied to many areas, e.g., online advertising \citep{OnlineAdvertising2010}, clinical trials \citep{ClinicalTrials1995}, networking \citep{AdaptiveRouting2012,NetworkingLiu2017}, and pairwise ranking \citep{LimitedAdaptivity2017}. In this paper, we focus on stochastic multi-armed bandit. In this setting, each arm of the bandit is assumed to hold a distribution. Whenever the decision maker samples this arm, an independent instance of this distribution is returned. The decision maker adaptively chooses some arms to sample in order to achieve some specific goals. So far, the majority of works in this area has been focused on minimizing the \textsl{regret} (deviation from optimum), (e.g., \citep{UCBRegret2002,UCBRevisited2010,ChernoffInformation2011,AdaptiveRouting2012,TompsonSampling2012,UCBoost2018}) that addresses the trade-off between the exploration and exploitation of arms to minimize the regret.

Instead of regret minimization, this paper focuses on pure exploration problems, which aim either (i) to identify one or multiple arms satisfying specific conditions (e.g., with the highest expected rewards) and try to minimize the number of samples taken (e.g., \citep{FKLowerBound2004,Halving2010,FULowerBound2012,OnTopK2015,LimitedAdaptivity2017,KLLUCB2013,Q-IK2013,Q-IU2017,QuantileExploration2018}), or (ii) to identify one or multiple best possible arms according to a given criteria within a fixed number of samples (e.g., \citep{LimitedRounds2010,SimpleRegret2011,SimpleRegret2015}). In some applications such as product testing \citep{ProductTesting2009,LimitedRounds2010,ProductTesting2010}, before the products are launched, rewards are insignificant, and it is
more interesting to explore the best products with the
least cost, which suggests the pure exploration setting. This paper focuses on (i) above.

We investigate the problem of identifying any $k$ arms that are in the top $\rho$ fraction of the expected rewards of the arm set. This is in contrast to most works in the pure exploration paradigm that focus on the problem of identifying $k$ best arms of a given arm set. We name the former as the ``quantile exploration" (QE) problem, and the latter as the ``$k$-exploration" (KE) problem, respectively. The motivations of studying the QE problem are as follows: First, in many applications, it is not necessary to identify the best arms, since it is acceptable to find ``good enough'' arms. For instance, a company wants to hire 100 employees from more than 10,000 applicants. It may be costly to find the best 100 applicants, and may be acceptable to identify 100 within a certain top percentage (e.g., 5\%); Second, theoretical analysis \citep{FULowerBound2012,FKLowerBound2004} shows that the lower bound on the sample complexity (aka, number of samples taken) of the KE problem depends on $n$. When the number of arms is extremely large or possibly infinite (e.g., the problem of finding a ``good" size from $[10mm,12mm]$ for a new model of devices), it is infeasible to find the best arms, but may be feasible to find arms within a certain top quantile; Third, by adopting the QE setting, we replace the sample complexity's dependence on $n$ of the KE problem with ${k}/{\rho}$ \citep{Q-IU2017}, which can be much smaller, and can greatly reduce the number of samples needed to find ``good" arms.

This paper adopts the probably approximately correct (PAC) setting, where an $\epsilon$ bounded error is tolerated. This setting can avoid the cases where arms are too close---making the number of samples needed extremely large. The PAC setting has been adopted by numerous previous works \citep{FKLowerBound2004,FULowerBound2012,Halving2010,OnTopK2015,Q-IK2013,Q-IU2017,QuantileExploration2018,KLLUCB2013}.

\textbf{Model and Notations:} Let $\mathcal{S}$ be the set of arms. It can be finite or infinite. When $\mathcal{S}$ is finite, let $n$ be its size, and the top $\rho$ fraction arms are simply the top $\lfloor\rho n\rfloor$ arms. If $\mathcal{S}$ is infinite, we assume that the arms' expected rewards follow some unknown prior identified by an unknown cumulative distribution function (CDF) $\mathcal{F}$. $\mathcal{F}$ is not necessarily continuous. In this paper, we assume the rewards of the arms are of the same finite support, and normalize them into $[0,1]$. For an arm $a$, we use $R^t_a$ to denote the reward of its $t$-th sample. $(R^t_a,t\in\mathbb{Z}^+)$ are identical and independent. We also assume that the samples are independent across time and arms. For any arm $a$, let $\mu_a$ be its expected reward, i.e., $\mu_a:=\mathbb{E}R^1_a$. To formulate the problem, for any $\rho\in(0,1)$, we define the inverse of $\mathcal{F}$ as
\begin{align}\label{ICDF}
\mathcal{F}^{-1}(p):=\sup\{x:\mathcal{F}(x)\leq p\}.
\end{align} 
The inverse $\mathcal{F}^{-1}$ has the following two properties (\ref{ICDF1}) and (\ref{ICDF2}), where $X\sim\mathcal{F}$ means that $X$ is a random variable following the distribution defined by $\mathcal{F}$.
\begin{gather}
\mathcal{F}(\mathcal{F}^{-1}(p))\geq p,\label{ICDF1}\\
\mathbb{P}_{X\sim\mathcal{F}}\{X\geq \mathcal{F}^{-1}(p)\}\geq 1-p.\label{ICDF2}
\end{gather}
To see (\ref{ICDF1}), by contradiction, suppose $\mathcal{F}(\mathcal{F}^{-1}(p))<p$. Since $\mathcal{F}(x)$ is right continuous, there exists a number $x_1$ such that $x_1>\mathcal{F}^{-1}(p)$ and $\mathcal{F}(x_1)<p$. This implies that $x_1$ is in $\{x:\mathcal{F}(x)\leq p\}$, and thus contradicting (\ref{ICDF}). Define $\mathcal{G}(x):=\mathbb{P}_{X\sim\mathcal{F}}\{X\geq x\}$. Similar to (\ref{ICDF1}), the left continuity of $\mathcal{G}$ implies (\ref{ICDF2}).

In the finite-armed case, an arm $a$ is said to be $(\epsilon, m)$-optimal if $\mu_a+\epsilon\geq\lambda_{[m]}$, where $\lambda_{[m]}$ is defined as the $m$-th largest expected reward among all arms in $\mathcal{S}$. In other words, the expected reward of an $(\epsilon, m)$-optimal arm plus $\epsilon$ is no less than $\lambda_{[m]}$. The QE problem is to find $k$ distinct $(\epsilon, m)$-optimal arms of $\mathcal{S}$. We consider both cases where $\lambda_{[n]}$ is known and unknown. 

Given a set $\mathcal{S}$ of size $n$, $k\in\mathbb{Z}^+$ and $\epsilon,\delta\in(0,\frac{1}{2})$, we define the two finite-armed QE problems Q-FK (Quantile, Finite-armed, $\lambda_{[m]}$ Known) and Q-FU (Quantile, Finite-armed, $\lambda_{[m]}$ Unknown) as follows:
\begin{problem}[\textbf{Q-FK}]\label{Q-FK}
	With known $\lambda_{[m]}$, we want to find $k$ distinct $(\epsilon, m)$-optimal arms with at most $\delta$ error probability, and use as few samples as possible.
\end{problem}
\begin{problem}[\textbf{Q-FU}]\label{Q-FU}
	Without knowing $\lambda_{[m]}$, we want to find $k$ distinct $(\epsilon, m)$-optimal arms with at most $\delta$ error probability, and use as few samples as possible.
\end{problem}

In the infinite-armed case, an arm is said to be $[\epsilon, \rho]$-optimal if its expected reward is no less than $\mathcal{F}^{-1}(1-\rho)-\epsilon$. Here, we use brackets to avoid ambiguity. To simplify notation, we define $\lambda_\rho:=\mathcal{F}^{-1}(1-\rho)$. An $[\epsilon,\rho]$-optimal arm is within the top $\rho$ fraction of $\mathcal{S}$ with an at most $\epsilon$ error. We consider both cases where $\lambda_{\rho}$ is known and unknown. Note that in both cases, we have no knowledge on $\mathcal{F}$ except that $\lambda_\rho$ is possibly known.

Given a set $\mathcal{S}$ of infinite number of arms, $k\in\mathbb{Z}^+$, and $\rho, \delta, \epsilon \in (0,1/2)$, we define the two infinite-armed QE problems Q-IK (Quantile, Infinite-armed, $\lambda_\rho$ Known) and Q-IU (Quantile, Infinite-armed, $\lambda_\rho$ Unknown).
\begin{problem}[\textbf{Q-IK}]\label{Q-IK}
	Knowing $\lambda_\rho$, we want to find $k$ distinct $[\epsilon, \rho]$-optimal arms with error probability no more than $\delta$, and use as few samples as possible.
\end{problem}
\begin{problem}[\textbf{Q-IU}]\label{Q-IU}
	Without knowing $\lambda_\rho$, we want to find $k$ distinct $[\epsilon, \rho]$-optimal arms with error probability no more than $\delta$, and use as few samples as possible.
\end{problem}

\section{RELATED WORKS}\label{sec:RW}
\begin{table*}[t]
	\caption{Comparison of Previous Works and Ours. All Bounds Are for the Worst Instances.}\label{tabasdf}
	\centering
	\begin{tabular}{lll}
		\multicolumn{1}{l}{\textbf{PROBLEM}} & \multicolumn{1}{l}{\textbf{WORK}} & \multicolumn{1}{l}{\textbf{SAMPLE COMPLEXITY}}\\ 
		
		\hline
		\multirow{3}{*}{Q-IK} 
		&\multirow{2}{*}{\citet{Q-IK2013}} 	
		&\multicolumn{1}{l}{$O\left(\frac{1}{\rho\epsilon^2}\log\frac{1}{\delta}\right)$ for $k=1$} \\ 
		& &\multicolumn{1}{l}{$\Omega\left(\frac{1}{\epsilon^2}\left(\frac{1}{\rho}+\log\frac{1}{\delta}\right)\right)$ for $k=1$} \\
		\cline{2-3}
		&\multicolumn{1}{l}{\textbf{This Paper}} 	
		&\multicolumn{1}{l}{$\Theta\left(\frac{k}{\epsilon^2}\left(\frac{1}{\rho}+\log\frac{k}{\delta}\right)\right)$ for $k\in\mathbb{Z}^+$} \\

		\hline
		\multirow{3}{*}{Q-FK} 
		&\multirow{1}{*}{\citet{Q-IK2013}} 	
		&\multicolumn{1}{l}{$O\left(\frac{m}{n\epsilon^2}\log\frac{1}{\delta}\right)$ for $k=1$} \\ 
		\cline{2-3}
		&\multirow{2}{*}{\textbf{This Paper}} 
		&\multicolumn{1}{l}{$O\left(\frac{1}{\epsilon^2}\left(n\log\frac{m+1}{m+1-k}+k\log\frac{k}{\delta}\right)\right)$ for $k\leq m\leq n/2$}\\
		& &\multicolumn{1}{l}{$\Omega\left(\frac{k}{\epsilon^2}\left(\frac{n}{m}+\log\frac{k}{\delta}\right)\right)$ for $k\leq m\leq n/2$} \\ 
		
		\hline
		\multirow{3}{*}{Q-IU} 
		&Chaudhuri et al. (2017) & $O\left(\frac{1}{\rho\epsilon^2}\log^2\frac{1}{\delta}\right)$ for $k=1$\\
		& and \citet{QuantileExploration2018} & $\Omega\left(\frac{1}{\rho\epsilon^2}\log\frac{1}{\delta}\right)$ for $k=1$\\
		\cline{2-3}
		&\multirow{2}{*}{\textbf{This Paper}} 	
		&\multicolumn{1}{l}{$O\left(\frac{1}{\epsilon^2}\left(\frac{1}{\rho}\log^2\frac{1}{\delta}+ k \left(\frac{1}{\rho} + \log\frac{k}{\delta} \right) \right)\right)$ for $k\in\mathbb{Z}^+$} \\
		& &\multicolumn{1}{l}{$\Omega\left(\frac{1}{\epsilon^2}\left(\frac{1}{\rho}\log\frac{1}{\delta}+ k\left(\frac{1}{\rho} + \log\frac{k}{\delta}\right)\right)\right)$ for $k\in\mathbb{Z}^+$} \\
		
		\hline
		\multirow{5}{*}{Q-FU} 
		&\multirow{2}{*}{Chaudhuri et al. (2017)} & $O\left(\frac{n}{m\epsilon^2}\log^2\frac{1}{\delta}\right)$ for $k=1$\\
		& & $\Omega\left(\frac{n}{m\epsilon^2}\log\frac{1}{\delta}\right)$ for $k=1$\\
		\cline{2-3}
		
		&\multicolumn{1}{l}{\citet{QuantileExploration2018}} & $O\left(\frac{n}{m\epsilon^2}\log^2\frac{1}{\delta}\right)$ for $k=1$\\
		\cline{2-3}
		
		&\multirow{2}{*}{\textbf{This Paper}} 	
		&\multicolumn{1}{l}{$O\left(\frac{1}{\epsilon^2}\left(\frac{n}{m}\log^2\frac{1}{\delta}+n\log\frac{m+2}{m+2-2k}+k\log\frac{k}{\delta}\right)\right)$ for $2k< m\leq n/2$} \\ 
		& &\multicolumn{1}{l}{$\Omega\left(\frac{1}{\epsilon^2}\left(\frac{n}{m}\log\frac{1}{\delta}+ k\left(\frac{n}{m} + \log\frac{k}{\delta}\right)\right)\right)$ for $k\leq m\leq n/2$} \\ 
		\hline
	\end{tabular}
\end{table*}

To our best knowledge, \citet{Q-IK2013} was the first one who has focused on the QE problems. They derived the tight lower bound $\Omega(\frac{1}{\epsilon^2}(\frac{1}{\rho}+\log\frac{1}{\delta}))$\footnote{All $\log$, unless explicitly noted, are natural $\log$.} for the Q-IK problem with $k=1$. They also provided a Q-IK algorithm for $k=1$, with sample complexity $O(\frac{1}{\rho\epsilon^2}\log\frac{1}{\delta})$, higher than the lower bound roughly by a $\log\frac{1}{\delta}$ factor. In contrast, our Q-IK algorithm works for all $k$-values and matches the lower bound.

\citet{Q-IU2017} studied the Q-IU and Q-FU problems with $k=1$. They derived the lower bounds for $k=1$. In this paper, we generalize their lower bounds to cases where $k>1$. They also proposed algorithms for these two problems with $k=1$, and the upper bounds ($O(\frac{1}{\rho\epsilon^2}\log^2\frac{1}{\delta})$ for Q-IK, $O(\frac{n}{m\epsilon^2}\log^2\frac{1}{\delta})$ for Q-FK) are the same as ours. For $k>1$, by simply repeating their algorithms $k$ times and setting error probability $\frac{\delta}{k}$ for each repetition, one can solve the two problems with sample complexity $O(\frac{k}{\rho\epsilon^2}\log^2\frac{k}{\delta})$ and $O(\frac{n}{mk\epsilon^2}\log^2\frac{k}{\delta})$, respectively. This paper proposes new algorithms for all $k$-values with $\log\frac{k}{\delta}$ reductions over the sample complexity.

\citet{QuantileExploration2018} studied the Q-IU problem. They proposed a Q-IK algorithm which is higher than the lower bound proved in this paper by a $\log\frac{1}{\rho\delta}$ factor in the worst case. Under some ``good" priors, its theoretical sample complexity can be lower than ours. However, numerical results in this paper show that our algorithm still obtains improvement under ``good" priors.

Although the KE problem is not the focus of this paper, we provide a quick overview for comparative perspective. An early attempt on the KE problem was done by \citet{MedianElimination2002}, which proposed an algorithm called Median-Elimination that finds an $(\epsilon,1)$-optimal arm with probability at least $1-\delta$ by using at most $O(\frac{n}{\epsilon^2}\log\frac{1}{\delta})$ samples. \citet{FKLowerBound2004,FULowerBound2012,Halving2010,LimitedAdaptivity2017,OnTopK2015,LIL2014,GCL2016,KLLUCB2013} studied the KE problem in different settings. Halving algorithm proposed by \citet{Halving2010} finds $k$ distinct $(\epsilon,k)$-optimal arms with probability at least $1-\delta$ by using $O(\frac{n}{\epsilon^2}\log\frac{k}{\delta})$ samples. \citet{Halving2010,FULowerBound2012,LIL2014,Q-IU2017,QuantileExploration2018,KLLUCB2013} used confidence bounds to establish algorithms that can exploit the large gaps between the arms. In practice, these algorithms are promising, while in theory, their sample complexities may be higher than the lower bound by $\log$ factors.

\section{LOWER BOUND ANALYSIS}\label{sec:LBA}
We first establish the Q-FK lower bound.
\begin{theorem}[Lower bound for Q-FK]\label{LB-Q-FK}
	Given $k\leq m\leq n/2$, $\epsilon\in(0,\frac{1}{4})$, and $\delta\in(0,e^{-8}/40)$, there is a set such that to find $k$ distinct $(\epsilon, m)$-optimal arms of it with error probability at most $\delta$, any algorithm must take $\Omega(\frac{k}{\epsilon^2}(\frac{n}{m} + \log{\frac{k}\delta}))$ samples in expectation.
\end{theorem}
\textit{Proof Sketch.}
\citet[Theorem~13]{FKLowerBound2004} showed that for the worst instance, finding an $(\epsilon, 1)$-optimal arm with confidence $1-\delta$ needs $\Omega(\frac{1}{\epsilon^2}(\frac{n}{m}+\log\frac{1}{\delta}))$ samples in expectation. We show that any Q-FK algorithm for $k=1$ can be transformed to find $(\epsilon, 1)$-optimal arms, and derive the desired lower bound for $k=1$. Then, we construct a series of Q-IK problems with $k=1$ that match the lower bound proved above. We show that solving any $k$ of these problems with at most $\delta$ total error probability needs at least $\Omega(\frac{k}{\epsilon^2}(\frac{n}{m} + \log{\frac{k}\delta}))$ samples in expectation. Any algorithm that solves the Q-FK problem with parameter $k$ can be transformed to solve the above problems. The desired lower bound follows.
$\square$

By Theorem~\ref{LB-Q-FK}, we establish the Q-IK lower bound.
\begin{theorem}[Lower bound for Q-IK]\label{LB-Q-IK}
	Given $k$, $\rho\in(0,\frac{1}{2}]$, $\epsilon\in(0,\frac{1}{4})$, and $\delta\in(0,e^{-8}/40)$, there is an infinite set such that to find $k$ distinct $[\epsilon, \rho]$-optimal arms of it with error probability at most $\delta$, any algorithm must take $\Omega(\frac{k}{\epsilon^2}(\frac{1}{\rho}+\log{\frac{k}\delta}))$ samples in expectation.
\end{theorem}\textit{Proof.}
By contradiction, suppose there is an algorithm $\mathcal{A}$ that solves all instances of the Q-IK problem by using $o(\frac{k}{\epsilon^2}(\frac{1}{\rho}+\log{\frac{k}\delta}))$ samples in expectation. Choosing $m\geq {k(k-1)}/{\delta}$ and $n\geq 2m$, we construct an $n$-sized set $\mathcal{C}$ that meets the lower bound of the Q-FK problem. By drawing arms from $\mathcal{C}$ with replacement, we can apply $\mathcal{A}$ to it with $\rho=\frac{m}{n}$. Now, we use $\mathcal{A}$ to find $k$ possibly duplicated $(\epsilon, m)$-optimal arms of $\mathcal{C}$ with error probability $\delta/2$. The probability that there is no duplication in these $k$ found arms is at least $\prod_{i=1}^k{\frac{m+1-i}{m}}\geq 1-\sum_{i=1}^k{\frac{i-1}{m}}\geq 1-\frac{\delta}{2}$. Thus, with probability at least $1-\delta$, $\mathcal{A}$ finds $k$ distinct $(\epsilon, m)$-optimal arms of $\mathcal{C}$ by $ o(\frac{k}{\epsilon^2}(\frac{n}{m}+\log{\frac{k}\delta}))$ samples in expectation, contradicting Theorem~\ref{LB-Q-FK}. The proof is complete.
$\square$

The lower bound for the Q-FU problem directly follows from Theorem~3.3 of \citep{Q-IU2017} and Theorem~\ref{LB-Q-FK}. Theorem~3.3 \citep{Q-IU2017} gives an $\Omega(\frac{n}{m\epsilon^2}\log\frac{1}{\delta})$ lower bounds for $k=1$. Corollary~\ref{LB-Q-FU} applies for all $k$. 
\begin{corollary}[Lower bound for Q-FU]\label{LB-Q-FU}
	Given $k\leq m\leq n/2$, $\epsilon\in(0,1/\sqrt{32})$, and $\delta\in(0,e^{-8}/40)$, there is a set such that to find $k$ distinct $(\epsilon, m)$-optimal arms with probability at least $1-\delta$, any algorithm must take $\Omega(\frac{1}{\epsilon^2}(\frac{nk}{m} + k\log{\frac{k}\delta}+\frac{n}{m}\log\frac{1}{\delta}))$ samples in expectation.
\end{corollary}

The lower bound for the Q-FU problem directly follows from Corollary 3.4 of \citep{Q-IU2017} and Theorem~\ref{LB-Q-IK}. Corollary 3.4 of \citet{Q-IU2017} gives an $\Omega(\frac{1}{\rho\epsilon^2}\log\frac{1}{\delta})$ lower bound for $k=1$. Corollary~\ref{LB-Q-IU} applies for all $k$.
\begin{corollary}[Lower bound for Q-IU]\label{LB-Q-IU}
	Given $k$, $\rho\in(0,\frac{1}{2}]$, $\epsilon\in(0,1/\sqrt{32})$, and $\delta\in(0,e^{-8}/40)$, there is an infinite set such that to find $k$ distinct $[\epsilon, \rho]$-optimal arms with probability at least $1-\delta$, any algorithm must take $\Omega(\frac{1}{\epsilon^2}(\frac{k}{\rho}+k\log{\frac{k}\delta}+\frac{1}{\rho}\log\frac{1}{\delta}))$ samples in expectation.
\end{corollary}

\section{ALGORITHMS FOR THE Q-IK PROBLEM}\label{sec:Q-IK}
In this section, we present two Q-IK algorithms: AL-Q-IK and CB-AL-Q-IK. ``AL" stands for ``algorithm" and ``CB" stands for ``confidence bounds".

\textbf{A worst case order-optimal algorithm.} We first introduce AL-Q-IK. It calls the function ``Median-Elimination" \citep{MedianElimination2002}, which finds an $(\epsilon, 1)$-optimal arm with probability at least $1-\delta$ by using $O(\frac{|A|}{\epsilon^2}\log\frac{1}{\delta})$ samples. AL-Q-IK is similar to Iterative Uniform Rejection (IUR) \citep{Q-IK2013}. At each repetition, IUR draws an arm from $\mathcal{S}$, performs $\Theta(\frac{1}{\epsilon^2}\log\frac{1}{\delta})$ samples on it, and returns it if the empirical mean is large enough. It solves the Q-IK problem with $k=1$, and its sample complexity is $O(\frac{1}{\epsilon^2\rho}\log\frac{1}{\delta\rho})$. This is higher than the lower bound roughly by a $\frac{1}{\rho}\log\frac{1}{\rho}$ factor (compared with the $\Omega(\frac{1}{\epsilon^2}\log\frac{1}{\delta})$ term). The $\frac{1}{\rho}\log\frac{1}{\rho}$ factor is due to the fact that the random arm drawn from $\mathcal{S}$ is $[\epsilon,\rho]$-optimal with probability $\rho$ (in the worst case). Inspired by their work, we add Lines 2 and 3 to ensure that $a_t$ is $[\epsilon_1,\rho]$-optimal with probability at least $\frac{1}{2}$. By doing this, we replace the $\frac{1}{\rho}\log\frac{1}{\rho}$ factor by a constant while adding $O(\frac{1}{\rho\epsilon^2})$ samples for each repetition. Repetitions continue until $k$ arms are found, and the number of repetitions is no more than $4k$ in expectation. The choice of $n_2$ guarantees that for each arm added to $Ans$, it is $[\epsilon,\rho]$-optimal with probability at least $1-\frac{\delta}{k}$. We state its theoretical performance in Theorem~\ref{TP-AL-Q-IK}.

\begin{algorithm}[h]
	\caption{AL-Q-IK$(\mathcal{S}, k, \rho, \epsilon, \delta, \lambda)$}\label{AL-Q-IK}
	\hspace*{\algorithmicindent} \textbf{Input:} $\mathcal{S}, k, \rho, \epsilon, \delta$, and $\lambda\leq\mathcal{F}^{-1}(1-\rho)$;\\
	\hspace*{\algorithmicindent} \textbf{Initialize:} Choose $\epsilon_1,\epsilon_2>0$ with $\epsilon_1+2\epsilon_2=\epsilon$; $t\gets 0$; $Ans\gets\emptyset$; $n_1\gets \lceil \frac{1}{\rho}\log{3}\rceil$; $n_2\gets \lceil \frac{1}{2\epsilon_2^2}\log\frac{k}{\delta}\rceil$;\\\indent\Comment{$\epsilon_1,\epsilon_2=\Omega(\epsilon)$, $Ans$ stores the chosen arms;}
	\begin{algorithmic}[1]
		\Repeat \ $t\gets t + 1$;
		\State Draw $n_1$ arms from $\mathcal{S}$, and form set $A_t$;
		\State arm $a_t\gets$Median-Elimination$(A_t, \epsilon_1, \frac{1}{4})$;
		\State Sample $a_t$ for $n_2$ times;
		\State $\hat{\mu}_t\gets$ the empirical mean;
		\If{$\hat{\mu}_t\geq \lambda_\rho-\epsilon_1-\epsilon_2$}
		\State $Ans\gets Ans\cup \{a_t\}$;
		\EndIf
		\Until{$|Ans|\geq k$} 
		\State \Return $Ans$;
	\end{algorithmic}
\end{algorithm}
\begin{theorem}[Theoretical performance of AL-Q-IK]\label{TP-AL-Q-IK}
	With probability at least $1-\delta$, AL-Q-IK returns $k$ distinct arms having expected rewards no less than $\lambda-\epsilon$. The expected sample complexity is $O(\frac{k}{\epsilon^2}(\frac{1}{\rho}+\log\frac{k}{\delta}))$.
\end{theorem}
\textit{Proof Sketch.}
\textit{Correctness:} Here, we note that $\lambda_\rho\geq \lambda$. At each repetition, $n_1$ arms are drawn from $\mathcal{S}$ to guarantee that with probability at least $2/3$, the set $A_t$ contains an arm of the top $\rho$ fraction. Then in Line~3, Median-Elimination$(A_t, \epsilon_1, \frac{1}{4})$ is called to get an $a_t$, which is $[\epsilon_1, \rho]$-optimal with probability at least $\frac{2}{3}(1-\frac{1}{4})=\frac{1}{2}$. At Line~5, by Hoeffding's Inequality, we can prove that if $a_t$ is $[\epsilon_1, \rho]$-optimal, $\hat{\mu}_t$ is greater than $\lambda - \epsilon_1-\epsilon_2$ with probability at least $1-\frac{\delta}{k}$, and if $\mu_{a_t}\leq \lambda-\epsilon$, $\hat{\mu}_t$ is less than $\lambda - \epsilon_1-\epsilon_2$ with probability at least $1-\frac{\delta}{k}$. By some computation, we show that given $a_t$ is added to $Ans$, $\mu_{a_t}\geq \lambda-\epsilon$ with probability at least $1 - \frac{\delta}{k}$. Thus, with probability at least $1-\delta$, all arms in $Ans$ having expected rewards $\geq \lambda-\epsilon$. \textit{Sample Complexity:} For each $t$, $a_t$ is $[\epsilon_1,\rho]$-optimal with probability at least $\frac{1}{2}$, and if $a_t$ is $[\epsilon_1, \rho]$-optimal, then with probability at least $1-\frac{\delta}{k}$, it will be added to $Ans$. Thus, in the $t$-th repetition, with probability at least $(1-\frac{\delta}{k})\cdot\frac{1}{2} \geq \frac{1}{4}$, one arm is added to $Ans$. Thus, the algorithm returns after average $4k$ repetitions. In each repetition, Line~3 takes $O(\frac{n_1}{\epsilon^2}\log{4})=O(\frac{1}{\rho\epsilon^2})$ samples, and Line~4 takes $n_2=O({\epsilon^{-2}}\log(k/\delta))$ samples, proving the sample complexity.
$\square$

\textbf{Remark:} The expected sample complexity of Algorithm~\ref{AL-Q-IK} matches the lower bound proved in Theorem~\ref{LB-Q-IK}. Even for $k=1$, this result is better than the previous works $O(\frac{1}{\rho\epsilon^2}\log\frac{1}{\delta})$ \citep{Q-IK2013}.

\textbf{Alternative Version Using Confidence Bounds}
AL-Q-IK is order-optimal for the worst instances, and provides theoretical insights on the Q-IK problem, but in practice, it does not exploit the large gaps between the arms' expected rewards. In this part, we use confidence bounds to establish an algorithm that is not order-optimal for the worst instance but has better practical performance for most instances. Many previous works \citep{Halving2010,FULowerBound2012,LIL2014,Q-IU2017,QuantileExploration2018} have shown that this type of confidence-bound-based (CBB) algorithms can dramatically reduce the actual number of samples taken in practice. Given an arbitrary arm $a$ with expected reward $\mu_a$, we let $\hat{X}^N(a)$ be its empirical mean after $N$ samples. A function $u(\cdot)$ ($l(\cdot)$) is said to be an upper (lower) $\delta$-confidence bound if it satisfies
\begin{gather}
\mathbb{P}\{u(\hat{X}^N(a),N,\delta)\geq \mu_a\} \geq 1-\delta,\label{UpperConfidenceBound}\\
\mathbb{P}\{l(\hat{X}^N(a),N,\delta)\leq \mu_a\} \geq 1-\delta.\label{LowerConfidenceBound}
\end{gather}
There are many choices of confidence bounds, e.g., the confidence bounds using Hoeffding's Inequality can be
\begin{gather}
u(\hat{X}^N(a),N,\delta)=\hat{X}^N(a)+\sqrt{\log\delta^{-1}/(2N)},\label{HodffdingUpper}\\
l(\hat{X}^N(a),N,\delta)=\hat{X}^N(a)-\sqrt{\log\delta^{-1}/(2N)}.\label{HoefdingLower}
\end{gather}
In this paper, we propose a general algorithm that works for all confidence bounds satisfying (\ref{UpperConfidenceBound}) and (\ref{LowerConfidenceBound}). We first introduce PACMaxing (Algorithm~\ref{ALPACMaxing}), an algorithm to find one $(\epsilon,m)$-optimal arm. The idea follows KL-LUCB \citep{KLLUCB2013}, except that it is designed for all confidence bounds and has a budget to bound the number of samples taken. Adding $budget$ prevents the number of samples from blowing up to infinity, and helps establish Algorithm~\ref{CB-AL-Q-IK}.

In PACMaxing, we let $U^t(a):=u(\hat{\mu}^t(a),N^t(a),\delta^{N^t(a)})$ and $L^t(a):=l(\hat{\mu}^t(a),N^t(a),\delta^{N^t(a)})$. For every arm~$a$, PACMaxing guarantees that during the execution of the algorithm, with probability at least $1-\frac{\delta}{n}$, its expected reward is always between the lower and upper confidence bounds, and thus, is correct with probability at least $1-\delta$ (see Lemma~\ref{Correctness-PACMaxing}). Lemma~\ref{Correctness-PACMaxing}'s proof is similar to that of KL-LUCB \citep{KLLUCB2013}, and is provided in supplementary materials.

\begin{algorithm}[h]
	\caption{PACMaxing$(A, \epsilon, \delta, budegt)$}\label{ALPACMaxing}
	\hspace*{\algorithmicindent} \textbf{Input:} $A$ an $n$-sized set of arms; $\delta,\epsilon \in (0,1)$;
	\begin{algorithmic}[1]
		\State $\forall s, \delta^s:= \frac{\delta}{k_1ns^\gamma}$, where $\gamma>1$ and $k_1\geq2(1+\frac{1}{\gamma-1})$; 
		\State $t\gets 0$ (number of sample taken); 
		\State $B(t)\gets\infty$ (stopping index); 
		\State Sample every arm of $A$ once; $t\gets n$;
		\State $N^t(a)\gets 1$, $\forall a\in A$;(number of times $a$ is sampled)
		\State Let $\hat{\mu}^t(a)$ be the empirical mean of $a$;
		\State $a^t\gets \arg\max_{a}\hat{\mu}^t(a)$;
		\State $b^t\gets \arg\max_{a\neq a^t}{U^t(a)}$;
		\While{$B(t)>\epsilon\land t\leq budget$}
		\State Sample $a^t$ and $b^t$ once; $t\gets t+2$;
		\State Update $\hat{\mu}^t(a),\hat{\mu}^t(b),N^t(a),N^t(b)$; 
		\State Update $a^t$ and $b^t$ as Lines 7 and 8;
		\State $B(t)\gets U^t(b^t)-L^t(a^t)$;
		\EndWhile
		\If{$B(t)\leq \epsilon$}\quad\Return{$a^t$}
		\Else\quad\Return{a random arm}
		\EndIf
	\end{algorithmic}
\end{algorithm}
\begin{lemma}[Correctness of PACMaxing]\label{Correctness-PACMaxing}
	Given sufficiently large $budget$, PACMaxing returns an $(\epsilon, 1)$-optimal arm with probability at least $1-\delta$.
\end{lemma}

Lemma~\ref{Correctness-PACMaxing} does not provide any insight about PACMaxing's sample complexity because it depends on the confidence bounds we choose. For Hoeffding bounds defined by (\ref{HodffdingUpper}) and (\ref{HoefdingLower}), we give the sample complexity of PACMaxing in Lemma~\ref{SampleComplexityPACMaxing}. Its proof is relegated to supplementary materials due to space limitation. Here, we define $\Delta_b:=\frac{1}{2}\max\{\epsilon, \max_{a\in A}{\mu_a}-\mu_b\}$ for all arms $b$.
\begin{lemma}[Sample complexity of PACMaxing]\label{SampleComplexityPACMaxing}
	Using confidence bounds (\ref{HodffdingUpper}) and (\ref{HoefdingLower}), and for $budget$ no less than $3n+ \max\{\frac{8n}{\epsilon^2}\log\frac{k_1n}{\delta}, \frac{8(1+e^{-1})\gamma n}{\epsilon^2}\log\frac{4(1+e^{-1})\gamma}{\epsilon^2}\}$, with probability at least $1-\delta$, PACMaxing returns a correct result after $O(\sum_{a\in A}{\frac{1}{\Delta_a^2}\log\frac{n}{\delta\Delta_a}})$ samples.
\end{lemma}

Using PACMaxing, we establish the CBB version of AL-Q-IK, presented in Algorithm~\ref{CB-AL-Q-IK}. In the algorithm, we let $g_0$ and $g_1$ be the corresponding $budget$ lower bounds as in Lemma~\ref{SampleComplexityPACMaxing}. CB-AL-Q-IK is almost the same as AL-Q-IK, except that it replaces Median-Elimination and the sampling of $a_t$ by PACMaxing. 

\begin{algorithm}[h]
	\caption{CB-AL-Q-IK$(\mathcal{S}, k, \rho, \epsilon, \delta, \lambda)$}\label{CB-AL-Q-IK}
	\hspace*{\algorithmicindent} \textbf{Input:} $\mathcal{S}, k, \rho, \epsilon, \delta$, and $\lambda\leq\mathcal{F}^{-1}(1-\rho)$;\\
	\hspace*{\algorithmicindent} \textbf{Initialize:} Choose $\epsilon_1,\epsilon_2>0$ with $\epsilon_1+2\epsilon_2=\epsilon$; $t\gets 0$; $Ans\gets\emptyset$; $n_1\gets \lceil \frac{1}{\rho}\log{3}\rceil$; \\\indent\Comment{$\epsilon_1,\epsilon_2=\Omega(\epsilon)$, $Ans$ stores the chosen arms;}
	\begin{algorithmic}[1]
		\Repeat \ $t\gets t + 1$;
		\State Draw $n_1$ arms from $\mathcal{S}$, and form set $A_t$;
		\State arm $a_t\gets$PACMaxing$(A_t, \epsilon_1, 1/4,g_0)$;
		\State $c$ is an arm with constant rewards $\lambda-\epsilon_1-\epsilon_2$;
		\State $b_t\gets$PACMaxing$(\{a_t,c\},\epsilon_2,{\delta}/{k},g_1)$;
		\If{$b_t=a_t$}
		\State $Ans\gets Ans\cup\{a_t\}$;
		\EndIf
		\Until{$|Ans|\geq k$} 
		\State \Return $Ans$;
	\end{algorithmic}
\end{algorithm}
Theorem~\ref{TP-CB-AL-Q-IK} states the theoretical performance of CB-AL-Q-IK. Its worst case sample complexity is higher than the lower bound and that of AL-Q-IK roughly by a $\log\frac{1}{\rho\epsilon}$ factor. However, its empirical performance could be better (See Section~\ref{sec:NR} for numerical evidences).
\begin{theorem}[Theoretical performance of CB-AL-Q-IK]\label{TP-CB-AL-Q-IK}
	With probability at least $1-\delta$, CB-AL-Q-IK returns $k$ distinct arms having expected rewards no less than $\lambda-\epsilon$. When using confidence bounds (\ref{HodffdingUpper}) and (\ref{HoefdingLower}), it terminates after at most $O(\frac{k}{\epsilon^2}(\frac{1}{\rho}\log\frac{1}{\rho\epsilon}+\log\frac{k}{\delta\epsilon}))$ samples in expectation.
\end{theorem}
\textit{Proof.}
The correctness follows from directly using the same steps as in the proof of Theorem~\ref{TP-AL-Q-IK}. In each repetition, by Lemma~\ref{SampleComplexityPACMaxing}, the sample complexity of Line~3 is at most $O(\frac{n_1}{\epsilon^2}\log\frac{n_1}{\rho\epsilon})=O(\frac{1}{\rho\epsilon^2}\log\frac{1}{\rho\epsilon})$, and that of Line~5 is at most $O(\frac{1}{\epsilon^2}\log\frac{k}{\delta\epsilon})$. The ``at most" comes from the choice of $budget$ in Lemma~\ref{SampleComplexityPACMaxing}. The algorithm returns after at most $4k$ repetitions in expectation. The desired sample complexity follows.
$\square$

\section{ALGORITHMS FOR THE Q-IU PROBLEM}\label{sec:Q-IU}
\citet{Q-IU2017} proposed an $O(\frac{1}{\rho\epsilon^2}\log^2\frac{1}{\delta})$ sample complexity algorithm for $k=1$. Performing it for $k$ times with ${\delta}/{k}$ error probability for each can solve the problem for all $k$-values. However, this method will yield unnecessary dependency on $\log^2{k}$. If we can first estimate the value of $\lambda_\rho$, we can use (CB-)AL-Q-IK to solve this problem and replace the quadratic log dependency by $\log{k}$. We first use LambdaEstimation (LE) to get a ``good" estimation of $\lambda_\rho$, and then use AL-Q-IK to solve the Q-IU problem. We note that this idea may perform poorly for small $k$-values as evaluating $\lambda_\rho$ can take more samples than finding several $[\epsilon,\rho]$-optimal arms.

We first present the algorithm LE for estimating $\lambda_\rho$ in Algorithm~\ref{AL-LambdaEstimation}. LE calls Halving \citep{Halving2010}, which finds $k$ distinct $(\epsilon, k)$-optimal arms of an $n$-sized set with probability at least $1-\delta$ by taking $O(\frac{n}{\epsilon^2}\log\frac{k}{\delta})$ samples. Halving$_2$ is modified from Halving that finds PAC worst arms (defined by adding a minus to all rewards of the arms).

\begin{algorithm}[h]
	\caption{LambdaEstimation$(\mathcal{S},\rho,\epsilon,\delta)$ (LE)}\label{AL-LambdaEstimation}
	\hspace*{\algorithmicindent} \textbf{Input:} $\mathcal{S}$ an infinite set of arms; $\rho, \delta,\epsilon \in (0,1/2)$; 
	\begin{algorithmic}[1]
		\State Choose $\epsilon_1,\epsilon_2,\epsilon_3=\Omega(\epsilon)$ with $\epsilon_1+\epsilon_2+2\epsilon_3=\epsilon$;
		\State
		$n_3 \gets \lceil\frac{32}{\rho}\log\frac{5}{\delta}\rceil$;
		$n_4\gets \lceil \frac{1}{2\epsilon_3^2}\log\frac{10}{\delta}\rceil$;
		$m \gets \lfloor 1+\frac{3}{4}\rho n_3 \rfloor$;
		\State Draw $n_3$ arms from $\mathcal{S}$, and form $A_1$;
		\State $A_2\gets$ Halving$(A_1, m, \epsilon_1, \frac{\delta}{5})$;
		\State $\hat{a}\gets$ Halving$_2(A_2, 1, \epsilon_2, \frac{\delta}{5})$;
		\State Sample $\hat{a}$ for $n_4$ times, $\hat{\mu}_0\gets$the empirical mean;
		\State \Return $\hat{\lambda}\gets \hat{\mu}_0-\epsilon_2-\epsilon_3$;
	\end{algorithmic}
\end{algorithm}

In LE, we ensure that with probability at least $1-\frac{2}{5}\delta$, the $m$-th most rewarding arm of $A_1$ is in $M:=\{a\in\mathcal{S}:\lambda_\rho\leq\mu_a\leq\lambda_{\rho/2}\}$. After calling Halving and Halving$_2$, we get $\hat{a}$, whose expected reward is in $[\lambda_\rho-\epsilon_1, \lambda_{\rho/2}+\epsilon_2]$ with probability at least $1-\frac{4\delta}{5}$. Finally, $\hat{a}$ is sampled for $n_4$ times, and its empirical mean is in $[\lambda_\rho-\epsilon_1-\epsilon_3, \lambda_{\rho/2}+\epsilon_2+\epsilon_3]$ with probability at least $1-\delta$. Thus, the returned value $\hat{\lambda}$ is in $[\lambda_\rho-\epsilon, \lambda_{\rho/2}]$ with probability at least $1-\delta$. Detailed proof of Lemma~\ref{TPLE} can be found in supplementary materials.

\begin{lemma}[Theoretical performance of LE]\label{TPLE}
	After at most $O(\frac{1}{\rho\epsilon^2}\log^2\frac{1}{\delta})$ samples, LE returns $\hat{\lambda}$ that is in $[\lambda_\rho-\epsilon, \lambda_{\rho/2}]$ with probability at least $1-\delta$.
\end{lemma}
Now, we use LE to establish the Q-IU algorithm AL-Q-IU. Theorem~\ref{TP-AL-Q-IU} states its theoretical performance.

\begin{algorithm}[h]
	\caption{AL-Q-IU$(\mathcal{S}, k, \rho, \epsilon, \delta)$}\label{AL-Q-IU}
	\hspace*{\algorithmicindent} \textbf{Input:} $\mathcal{S}$ infinite; $k\in\mathbb{Z}^+$; $\rho, \delta,\epsilon \in (0,1/2)$;
	\begin{algorithmic}[1]
		\State $\hat{\lambda}\gets$ LE$(\mathcal{S},\rho,0.2\epsilon,\frac{\delta}{2})$;
		\State \Return AL-Q-IK$(\mathcal{S}, k, \frac{\rho}{2}, 0.8\epsilon, \frac{\delta}{2}, \hat{\lambda})$;
	\end{algorithmic}
\end{algorithm}
\begin{theorem}[Theoretical performance of AL-Q-IU]\label{TP-AL-Q-IU}
	With probability at least $1-\delta$, AL-Q-IU returns $k$ distinct $[\epsilon, \rho]$-optimal arms. With probability at least $1-\frac{\delta}{2}$, it terminates after $O(\frac{1}{\epsilon^2}(\frac{1}{\rho}\log^2\frac{1}{\delta}+ k(\frac{1}{\rho} + \log\frac{k}{\delta})))$ samples in expectation.
\end{theorem}
\textit{Proof.}
With probability at least $1-\frac{\delta}{2}$, $\hat{\lambda}$ is in $[\lambda_\rho-\frac{\epsilon}{2}, \lambda_{\rho/2}]$. When $\hat{\lambda}$ is in $[\lambda_\rho-\frac{\epsilon}{2}, \lambda_{\rho/2}]$, by Theorem~\ref{TP-AL-Q-IK}, Line~2 takes $O(\frac{k}{\epsilon^2}(\frac{1}{\rho}+\log\frac{1}{\delta}))$ samples in expectation, and, with probability at least $1-\frac{\delta}{2}$, all returned arms are $[\epsilon, \rho]$-optimal. The correctness of AL-Q-IU follows. The desired sample complexity follows by summing up $O(\frac{k}{\epsilon^2}(\frac{1}{\rho}+\log\frac{1}{\delta}))$ and $O(\frac{1}{\epsilon^2}\log^2\frac{1}{\delta})$ (Lemma~\ref{TPLE}).
$\square$

\textbf{Remark:} By Corollary~\ref{LB-Q-IU}, AL-Q-IU is sample complexity optimal up to a $\log\frac{1}{\delta}$ factor. When $\log\frac{1}{\delta} = O(k)$, i.e., $\delta\geq e^{-ck}$ for some constant $c>0$, AL-Q-IU is sample complexity optimal up to a constant factor.

\section{FINITE-ARMED ALGORITHMS}\label{sec:Finite}
In this section, we let $\mathcal{S}$ be a finite-sized set of arms. By drawing arms from it with replacement, these arms can be viewed as drawn from an infinite-sized set. We use $\mathcal{T(S)}$ to denote the corresponding infinite-sized set, and call it the \textsl{infinite extension} of $\mathcal{S}$.

\textbf{Q-FK.} When $k=1$, obviously, calling AL-Q-IK$(\mathcal{T(S)},1,\frac{m}{n},\epsilon,\delta,\lambda_\rho)$ can solve the Q-FK problem. When $k>1$, we can solve the Q-FK problem by repeatedly calling AL-Q-IK$(\mathcal{T(S)},1,\rho_t,\epsilon,\delta/k,\lambda_\rho)$ and updating $\mathcal{S}$ by deleting the chosen arm, where $\rho_t=\frac{m+1-t}{n+1-t}$. We present the algorithm AL-Q-FK (Algorithm for Q-FK) in Algorithm~\ref{AL-Q-FK}, and state the theoretical performance in Theorem~\ref{TP-AL-Q-FK}. The proof is relegated to supplementary materials.

\begin{algorithm}[h]
	\caption{AL-Q-FK$(\mathcal{S}, m, k, \epsilon, \delta, \lambda)$}\label{AL-Q-FK}
	\hspace*{\algorithmicindent} \textbf{Require:} $\mathcal{S}$ $n$-sized, $k\leq m\leq n/2$, $\lambda\leq\lambda_{[m]}$;\\
	\hspace*{\algorithmicindent} \textbf{Initialize:} $Ans\gets \emptyset$; \Comment{stores the chosen arms;}
	\begin{algorithmic}[1]
		\Repeat 
		\State $\mathcal{S}'\gets \mathcal{T}(\mathcal{S}-Ans)$;
		$\rho\gets \frac{m-|Ans|}{n-|Ans|}$;
		\State $a_t\gets$AL-Q-IK$(\mathcal{S}',1, \rho, \epsilon, {\delta}/{k}, \lambda)$;
		\State $Ans\gets Ans\cup\{a_t\}$;
		\Until{$|Ans| \geq k$}
		\State \Return{$Ans$};
	\end{algorithmic}
\end{algorithm}
\begin{theorem}[Theoretical performance of AL-Q-FK]\label{TP-AL-Q-FK}
	With probability at least $1-\delta$, AL-Q-FK returns $k$ distinct arms having mean rewards at least $\lambda-\epsilon$. Its takes $O(\frac{1}{\epsilon^2}(n\log\frac{m+1}{m+1-k}+k\log\frac{k}{\delta}))$ samples in expectation.
\end{theorem}

\textbf{Remark:} If $k\leq cm$ for some constant $c<1$, $\log\frac{m+1}{m+1-k}\leq \frac{k}{m+1-k} = O(\frac{k}{m})$, and thus, the expected sample complexity becomes $O(\frac{k}{\epsilon^2}(\frac{k}{m}+\log\frac{k}{\delta}))$, meeting the lower bound (Theorem~\ref{LB-Q-FK}). When $k$ is arbitrarily close to $m$, the Q-FK problem (almost) reduces to the KE problem. The tightest upper bound for the KE problem (with the knowledge of $\lambda_{[k]}$) is $O(\frac{n}{\epsilon^2}\log\frac{k}{\delta})$ \citep{FULowerBound2012} to our best knowledge. When $k$ is arbitrary close to $m$, as $O(\frac{1}{\epsilon^2}(n\log\frac{m+1}{m+1-k}+k\log\frac{k}{\delta}))=O(\frac{1}{\epsilon^2}(n\log{k}+k\log\frac{k}{\delta}))$, AL-Q-FK is still better than the literature asymptotically.

\textbf{Q-FU.} AL-Q-FU (Algorithm for solving the Q-FU problem) is presented in Algorithm~\ref{AL-Q-FU}. Its idea follows AL-Q-IU and AL-Q-FK. We only consider the case $k<\frac{m}{2}$. For $k\geq\frac{m}{2}$, it is better to use KE algorithms instead. Corollary~\ref{TP-AL-Q-FU} states its theoretical performance, which directly follows from Theorems~\ref{TP-AL-Q-IU} and \ref{TP-AL-Q-FK}.

\begin{algorithm}[h]
	\caption{AL-Q-FU$(\mathcal{S}, m, k, \epsilon, \delta)$}\label{AL-Q-FU}
	\hspace*{\algorithmicindent} \textbf{Require:} $\mathcal{S}$ $n$-sized; $2k < m\leq n/2$;
	\begin{algorithmic}[1]
		\State $\hat{\lambda}\gets$LE$(\mathcal{T(S)},\frac{m}{n},\frac{\epsilon}{2},\frac{\delta}{2})$;
		\State \Return{AL-Q-FU$(\mathcal{S},\lfloor\frac{m}{2}\rfloor,k,\frac{\epsilon}{2},\frac{\delta}{2})$};
	\end{algorithmic}
\end{algorithm}
\begin{corollary}[Theoretical Performance of AL-Q-FU]\label{TP-AL-Q-FU}
	With probability at least $1-\delta$, AL-Q-FU returns $k$ distinct $(\epsilon, m)$-optimal arms. With probability at least $1-\frac{\delta}{2}$, the expected number of samples AL-Q-FU takes is at most $O(\frac{1}{\epsilon^2}(\frac{n}{m}\log^2\frac{1}{\delta} +n\log\frac{m+2}{m+2-2k}+k\log\frac{k}{\delta}))$.
\end{corollary}

\textbf{Remark:} By Corollary~\ref{LB-Q-FU}, when $k\leq cm$ for some constant $c\in(0,\frac{1}{2})$, AL-Q-FU is sample complexity optimal up to a $\log\frac{1}{\delta}$ factor. If $\log\frac{1}{\delta} = O(k)$ also holds, i.e., $\delta\geq e^{-ck}$ for some constant $c>0$, then AL-Q-FU is sample complexity optimal in order sense.

\section{NUMERICAL RESULTS}\label{sec:NR}

In this section, we illustrate the improvements of our algorithms by running numerical experiments. We present the comparisons of CBB algorithms. The results of the non-CBB algorithms are presented in supplementary materials. Additional numerical results for the finite cases are also presented in supplementary materials. We first compare CB-AL-Q-IK with the literature, and then illustrate the comparison of CB-AL-Q-IU with previous works.

In the simulations, we adopt Bernoulli rewards for all the arms. For fair comparisons, for all CBB-algorithms or versions, we use the KL-Divergence based confidence bounds given by \citet{QuantileExploration2018} with $\gamma=2$. Every point in every figure is averaged over 100 independent trials. Previous works only considered the case where $k=1$. In the implementations, for $k>1$, we repeat them for $k$ times with $\frac{\delta}{k}$ error probability for each repetition.

First, we compare CBB algorithms for the Q-IK problem: CB-AL-Q-IK (choose $\epsilon_1=0.75\epsilon$) and $(\alpha,\epsilon)$-KL-LUCB \citep{QuantileExploration2018} (we name it KL-LUCB in this section). KL-LUCB is almost equivalent to $\mathcal{P}_2$ \citep{Q-IU2017} with a large enough batch size. The only difference is that they choose different confidence bounds. Here, we note that KL-LUCB does not require the knowledge of $\lambda_\rho$. However, we want to show that our algorithm along with this information can significantly reduce the actual number of samples needed. The priors $\mathcal{F}$ are all Uniform([0,1]). The results are summarized in Figure~\ref{fig:KLComparison} (a)-(d). 
\begin{figure}[bht]\centering
	\begin{subfigure}[b]{0.23\textwidth}
		\includegraphics[scale=0.5]{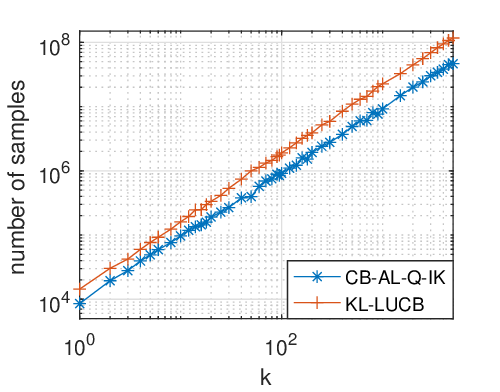}
		\caption{Vary $k$, $\rho=0.001$,$\epsilon=0.05$, and $\delta=0.001$.}
	\end{subfigure}
	\begin{subfigure}[b]{0.23\textwidth}
		\includegraphics[scale=0.5]{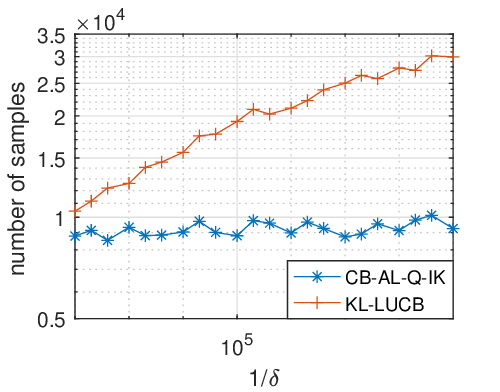}
		\caption{Vary $\delta$, $k=1$, $\rho=0.001$, and $\epsilon=0.05$.}
	\end{subfigure}\\
	\begin{subfigure}[b]{0.23\textwidth}
		\includegraphics[scale=0.5]{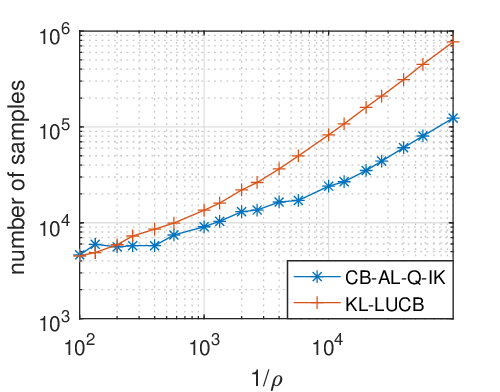}
		\caption{Vary $\rho$, $k=1$, $\epsilon=0.05$, and $\delta=0.001$.}
	\end{subfigure}
	\begin{subfigure}[b]{0.23\textwidth}
		\includegraphics[scale=0.5]{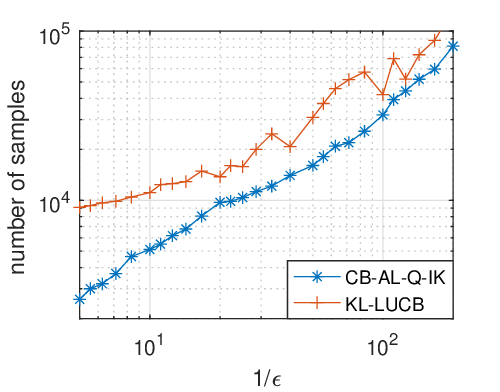}
		\caption{Vary $\epsilon$, $k=1$, $\rho=0.001$, and $\delta=0.001$.}
	\end{subfigure}
	\caption{Comparison of CB-AL-Q-IK and KL-LUCB.}\label{fig:KLComparison}
\end{figure}

It can be seen from Figure~\ref{fig:KLComparison} that CB-AL-Q-IK performs better than KL-LUCB except two or three points where $\rho$ is large. According to (a), the number of samples CB-AL-Q-IK takes increases slightly slower than KL-LUCB, consistent with the theory that CB-AL-Q-IK depends on $k\log{k}$ while KL-LUCB depends on $k\log^2{k}$. According to (b), we can see that KL-LUCB's number of samples increases obviously with $\frac{1}{\delta}$, while that of CB-AL-Q-IK is almost independent of $\delta$. The reason is that CB-AL-Q-IK depends on $(\frac{1}{\rho}\log\frac{1}{\rho}+\log\frac{1}{\delta})$ term, and when $\rho$ is small enough, $\log\frac{1}{\delta}$ can be dominated by $\frac{1}{\rho}\log\frac{1}{\rho}$. According to (c), CB-AL-Q-IK takes less samples than KL-LUCB for $\rho<0.005$, and the gap increases with $\frac{1}{\rho}$. According to (d), CB-AL-Q-IK performs better than KL-LUCB under the given $\epsilon$ values.

\begin{wrapfigure}{r}{0.23\textwidth}
	\centering
	\includegraphics[scale=0.5]{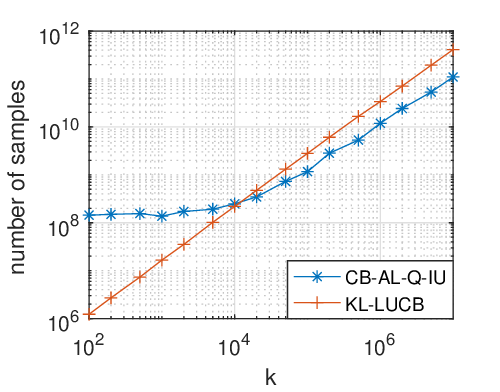}
	\caption{Comparison of CB-AL-Q-IU and KL-LUCB under prior $\mathcal{F}_h$. $\rho=0.05$, $\epsilon=0.1$, and $\delta=0.01$.}\label{fig:Q-IU}
\end{wrapfigure}  Second, we compare CB-AL-Q-IU and $(\alpha,\epsilon)$-KL-LUCB.
CB-AL-Q-IU is the CBB version of AL-Q-IU by replacing its subroutines by CBB ones. It is designed for large $k$-values and may not perform well under small $k$-values, even if it is always in order-sense better or equivalent compared to KL-LUCB.
The reason is that its subroutine LE has a large constant factor. However, since the sample complexities of these two algorithms both depend at least linearly on $k$ while that of LE is independent of $k$, 
when $k$ is large, the influence of (CB-)LE (the CBB version of LE) vanishes,	
and the improvement of (CB-)AL-Q-IK emerge. The results are summarized in Figure~\ref{fig:Q-IU}.  In the implementations, we choose $\epsilon_1=0.75\epsilon$ and $\epsilon_2=0.1\epsilon$ for algorithm LE, and choose $\epsilon_1=0.75\epsilon$ for CB-AL-Q-IK. In Figure~\ref{fig:Q-IU}, the algorithms are tested under a ``hard instance" $\mathcal{F}_h$, where $\rho$ fraction of the arms has expected reward $\frac{1}{2}+0.55\epsilon$ and the others have $\frac{1}{2}-0.55\epsilon$. The results are consistent with the theory, and suggest that CB-AL-Q-IK can use much less samples than KL-LUCB does when $k$ is sufficiently large.	

We admit that AL-Q-IU may not be practical as it takes $10^8$ samples even for $k=1$, but it also has several contributions: (i) It gives a hint for solving the Q-IU problem. If we can improve LE, we can get a practical algorithm for the Q-IU problem that works much better than the literature for large $k$-values. (ii) We can see from Figure~\ref{fig:Q-IU} that KL-LUCB increases faster as $k$ increases. It is consistent with the theory that KL-LUCB depends on $k\log^2{k}$ while (CB-)AL-Q-IU depends on $k\log{k}$. When $k$ is large enough, (CB)-AL-Q-IU can perform better. (iii) In order sense, the performance of (CB-)AL-Q-IU is better than the literature. Thus, our work gives better theoretical insights about the Q-IU problem.

\section{CONCLUSION}\label{sec:con}
In this paper, we studied the problems of finding $k$ top $\rho$ fraction arms with an $\epsilon$ bounded error from a finite or infinite arm set. We considered both cases where the thresholds (i.e., $\lambda_\rho$ and $\lambda_{[m]}$) are priorly known and unknown. We derived lower bounds on the sample complexity for all four settings, and proposed algorithms for them. For the Q-IK and Q-FK problems, our algorithms match the lower bounds. For the Q-IU and Q-FU problems, our algorithms are sample complexity optimal up to a log factor. Our simulations also confirm these improvements numerically.

\subsubsection*{Acknowledgements}
This work has been supported in part by NSF ECCS-1818791, CCF-1758736, CNS-1758757, CNS-1446582, CNS-1314538, CNS-1717060; ONR N00014-17-1-2417; AFRL FA8750-18-1-0107.

\bibliography{topmk}
\clearpage

\textbf{\scalebox{1.85}{Supplementary Materials}}
\section{PROOF OF THEOREM~\ref{LB-Q-FK}}
\textsl{Proof.} \textbf{For }$\mathbf{k=1}$. We first prove the lower bound for $k=1$.

\begin{claim}[Lower bound for Q-FK with $k=1$]\label{LB-Q-IK-K1}
	There is a priorly known $n$-sized set such that after randomly reordering, to find an $(\epsilon,m)$-optimal arm, any algorithm must use $\Omega(\frac{1}{\epsilon^2}(\frac{n}{m}+\log{\frac{1}\delta}))$ samples in expectation.
\end{claim}

\begin{proof}Let parameters $n$, $m$, $\epsilon$, and $\delta$ be given. For these parameters, by contradiction, suppose there is an algorithm $\mathcal{A}_1$ that solves every Q-FK instance with average sample complexity $o(\frac{1}{\epsilon^2}(\frac{n}{m} + \log{\frac{1}\delta}))$. To this end, we first introduce the following problem $\mathcal{P}_1$. 
	
	\textbf{Problem} $\mathcal{P}_1$: Given $\lfloor n/m \rfloor$ coins, where a toss of coin $i$ has an unknown probability $p_i$ to produce a head, and produce a tail otherwise. We name $p_i$ the ``head probability" of coin $i$. Let $p_{max}$ be the largest one among all $p_i$'s. Knowing the value of $p_{max}$, we want to find a coin whose head probability is no less than $p_{max}-\epsilon$, and the error probability is no more than $\delta$. 
	
	\citet[Theorem~13]{FKLowerBound2004} proved that the worst case sample complexity lower bound of $\mathcal{P}_1$ is $\Omega(\frac{1}{\epsilon^2}(\frac{n}{m} + \log{\frac{1}\delta}))$. Particularly, this lower bound can be met by the $\lfloor n/m \rfloor$-sized set $\{\frac{1}{2}+\epsilon,\frac{1}{2}-\epsilon,\frac{1}{2}-\epsilon,...,\frac{1}{2}-\epsilon\}$. Here, we will show that we can construct an algorithm from $\mathcal{A}_1$ that solves $\mathcal{P}_1$ with average sample complexity $o(\frac{1}{\epsilon^2}(\frac{n}{m} + \log{\frac{1}\delta}))$, implying a contradiction. 
	
	Let $\mathcal{C}_1$ be the set of the coins in $\mathcal{P}_1$. Before solving $\mathcal{P}_1$ by using $\mathcal{A}_1$, we need to do several operations over $\mathcal{C}_1$. We first define the ``duplication" of a coin. For each coin $i$, we ``duplicate" it for $m-1$ times and construct $m-1$ ``duplicated" coins such that whenever one wants to toss a duplication of coin $i$, coin $i$ will be tossed but the result is regarded as that of the duplication. Thus, we guarantee that all the duplications of coin $i$ have the same head probability as coin $i$. 
	
	With these duplications, we construct a new set $\mathcal{C}_2$ of coins with size $n$. $\mathcal{C}_2$ consists of all the coins of $\mathcal{C}_1$, all the duplications of all coins in set $\mathcal{C}_1$, and $(n-m\lfloor n/m \rfloor)$ \textit{negligible} coins (negligible coins are with head probability zero). Obviously, $\mathcal{C}_2$ consists of $n$ coins. In $\mathcal{P}_1$, for each head probability $p_i$, there are $m$ coins with head probability $p_i$ in $\mathcal{C}_2$. The negligible coins are used to make the size of $\mathcal{C}_2$ be $n$. 
	
	Then, we perform $\mathcal{A}_1$ on the set $\mathcal{C}_2$. It returns an $(\epsilon, m)$-optimal coin (coins can be regarded as arms with Bernoulli($p_i$) rewards) of $\mathcal{C}_2$ with probability at least $1-\delta$, and uses $o(\frac{1}{\epsilon^2}(\frac{1}{\rho}+\log\frac{1}{\delta}))$ samples in expectation. We use $c_r$ to denote the returned coin. Let coin $i^*$ be one of the coins whose head probability are $p_{max}$ (i.e., one of the most biased coins of $\mathcal{C}_1$). Since coin $i^*$ is duplicated for $m-1$ times, there are at least $m$ coins in $\mathcal{C}_2$ having head probability $p_{max}$. This implies that if $c_r$ is an $(\epsilon, m)$-optimal coin of $\mathcal{C}_2$, then its head probability is at least $p_{max}-\epsilon$. If $c_r$ is a negligible coin (i.e., with head probability zero), we return a random coin of $\mathcal{C}_1$ as the solution of $\mathcal{P}_1$. If $c_r$ is coin $i$ or one of its duplications, we return coin $i$ as the solution of $\mathcal{P}_1$. Noting that the negligible coins are not $(\epsilon,m)$-optimal, so if $c_r$ is an $(\epsilon,m)$-optimal coin of $\mathcal{C}_2$, there is a corresponding coin in $\mathcal{C}_1$ having the same probability as $c_r$. Thus, if $\mathcal{A}_1$ finds an $(\epsilon,m)$-coin of $\mathcal{C}_2$, it finds a coin of $\mathcal{C}_1$ whose head probability is at least $p_{max}-\epsilon$, which gives a correct solution to $\mathcal{P}_1$. To conclude, $\mathcal{A}_1$ solves $\mathcal{P}_1$ with average sample complexity $o(\frac{1}{\epsilon^2}(\frac{n}{m}+\log{\frac{1}\delta}))$, contradicting Theorem~13 \citep{FKLowerBound2004}. We note that we can choose $\mathcal{C}_1=\{\frac{1}{2}+\epsilon,\frac{1}{2}-\epsilon,...,\frac{1}{2}-\epsilon\}$ by \cite[Theorem~13]{FKLowerBound2004}, and thus, $\mathcal{C}_2$ is priorly known. This completes the proof of Claim~\ref{LB-Q-IK-K1}. 
\end{proof}

\textbf{For }$\mathbf{k>1}$. Now, we consider the case where $k>1$. From now on, we only consider the case where $m>2k$. For $m\leq 2k$, since the Q-IK problem does not become harder as $m$ increases, if the desired lower bound holds for $m>2k$, it also holds for $m\leq 2k$.

Let $\mathcal{C}_3$ be a priorly known $\lfloor \frac{n}{2k} \rfloor$-sized set such that after randomly reordering it, no algorithm can find one $(\epsilon, \lfloor \frac{m}{2k} \rfloor)$-optimal arm of it with probability $1-\delta$ by $o(\frac{1}{\epsilon^2}(\frac{n}{m} + \log{\frac{1}\delta}))$ samples in expectation, i.e., $\mathcal{C}_3$ meets the lower bound given in Claim~\ref{LB-Q-IK-K1}. Claim~\ref{LB-Q-IK-K1} guarantees that this set must exist. Choose a large enough positive integer $L$. By randomly reordering the indexes of arms in $\mathcal{C}_3$, we can construct $L$ sets that also meet the lower bound stated in Claim~\ref{LB-Q-IK-K1}. We refer to these sets as \textit{hard} sets. Now, we define problem $\mathcal{P}_2$ by these $L$ \textit{hard} sets.

\textbf{Problem} $\mathcal{P}_2$: Given the above $L$ \textit{hard} sets, we want to find $k$ distinct arms such that each of them is $(\epsilon, \lfloor \frac{m}{2k} \rfloor)$-optimal for a different \textit{hard} set, and the error probability is no more than $\delta$ in total.

\begin{claim}[Lower bound for $\mathcal{P}_2$]\label{LB-P2}
	To solve $\mathcal{P}_2$, at least $\Omega(\frac{k}{\epsilon^2}(\frac{n}{m}+\log\frac{k}{\delta}))$ samples are needed in expectation.
\end{claim}
\begin{proof}
	Let these $L$ hard instances be indexed by $1,2,...,L$. For each set $i$, by the definition of \textit{hard} sets, to find an $(\epsilon, \lfloor \frac{m}{2k} \rfloor)$-optimal arm from it with probability $1-\delta_i$, at least $\Omega(\frac{1}{\epsilon^2}(\frac{n}{m} + \log{\frac{1}\delta_i}))$ samples are needed in expectation. For an algorithm that solves $\mathcal{P}_2$, it returns $k$ arms, each of which belongs to a different \textit{hard} set. Without loss of generality, we say these $k$ returned arms belong to \textit{hard} sets $1,2,...,k$. Let $\delta_i$ denote the probability that the returned arm for \textit{hard} set $i$ is not $(\epsilon, \lfloor \frac{m}{2k} \rfloor)$-optimal. Obviously, to solve $\mathcal{P}_2$ with probability $1-\delta$, we need $\prod_{i=1}^k{(1-\delta_i)}\geq 1-\delta$. Further, since these sets are generated by randomly reordering a priorly known set $\mathcal{C}_3$, the samples of one set provide no information for the others. Thus, to solve $\mathcal{P}_2$, the expected sample complexity is at least
	\begin{align}\label{LB-KG1-sub}
	\Omega\left(\min\left\{\sum_{i=1}^k{\frac{1}{\epsilon^2}\log\frac{1}{\delta_i}}:\prod_{i=1}^k\left(1-\delta_i\right)\geq 1-\delta\right\}\right).
	\end{align}
	
	We note that the function $f(x)=\log(1/x)$ is convex, and thus, $\sum_{i=1}^k{\frac{1}{\epsilon^2}\log\frac{1}{\delta_i}}$ is convex over domain specified by the constraint $\prod_{i=1}^k\left(1-\delta_i\right)\geq 1-\delta$. Also, this constraint on $(\delta_i,i\in[k])$ is symmetric. By the property of convex functions, to get the minimal, we need to set $\delta_1=\delta_2=\cdots=\delta_k$. Thus, given $\prod_{i=1}^k\left(1-\delta_i\right)\geq 1-\delta$, we have
	\begin{align}\label{deltaLowerBound}
	\sum_{i=1}^k{\frac{1}{\epsilon^2}\log\frac{1}{\delta_i}} = \Omega\left(k\log\frac{k}{\delta}\right).
	\end{align}
	
	Applying Eq.~(\ref{deltaLowerBound}) to Eq.~(\ref{LB-KG1-sub}), we can get the desired lower bound. This completes the proof of Claim~\ref{LB-P2}.
\end{proof}

\begin{claim}\label{UP-A3}
	If there exists an algorithm $\mathcal{A}_2$ that can use $o(\frac{k}{\epsilon^2}(\frac{n}{m}+\log\frac{k}{\delta_1}))$ samples in expectation to find $k$ distinct $(\epsilon,m)$-optimal arms of any $n$-sized set with probability $1-\delta_1$ for $\delta_1\in(0,\delta]$, then we can construct another algorithm $\mathcal{A}_3$ that solves $\mathcal{P}_2$ by $o(\frac{k}{\epsilon^2}(\frac{n}{m}+\log\frac{k}{\delta}))$ samples in expectation.
\end{claim}
\begin{proof}
	We use $\mathcal{A}_2$ to construct a new algorithm $\mathcal{A}_3$, which works as follows: 
	
	Step~1): Pick $2k$ arbitrary \textit{hard} sets of coins (indexed by $1,2,...,2k$), and form a new set $\mathcal{C}_4$ (we note that coins in different sets are always considered to be different). Let $T=\lceil 2\log\frac{2k}{\delta}\rceil$. 
	
	Step~2): Performs algorithm $\mathcal{A}_2$ on $\mathcal{C}_4$ with error probability $\frac{\delta}{2T}$, and $\mathcal{A}_2$ returns $k$ arms. We refer to these returned arms as \textit{found} arms. 
	
	Step~3): For each \textit{found} arm, \textit{tag} the \textit{hard} set it belongs to. 
	
	Step~4): If at least $k$ \textit{hard} sets have been tagged, return one \textit{found} arm for each of the first $k$ tagged \textit{hard} set. Otherwise, go to Step~2. 
	
	We will prove that $\mathcal{A}_3$ solves $\mathcal{P}_2$ with expected sample complexity $o(\frac{k}{\epsilon^2}(\frac{n}{m}+\log\frac{k}{\delta}))$. 
	
	First we prove the correctness of $\mathcal{A}_3$. We note that for each \textit{hard} set $i$, the probability that an arbitrary \textit{found} arm belongs to it is $\frac{1}{2k}$. After $T$ calls of $\mathcal{A}_2$, there are $Tk$ \textit{found} arms, and thus, the probability that \textit{hard} set $i$ is not tagged is at most 
	\begin{align}
	\left(1-\frac{1}{2k}\right)^{Tk} \leq \left(1-\frac{1}{2k}\right)^{2k\log\frac{2k}{\delta}} \leq \frac{\delta}{2k}.
	\end{align}			
	
	Thus, with probability at least $1-\frac{\delta}{2}$, \textit{hard} sets $1,2,...,k$ are tagged after $T$ calls of $\mathcal{A}_2$. When a \textit{hard} set is tagged, at least one arm of it has been found by some call of $\mathcal{A}_2$. Also, each call is erred with probability at most $\frac{\delta}{2T}$. So, with probability at least $1-\frac{\delta}{2}$, the first $T$ calls of $\mathcal{A}_2$ all return correct results. Therefore, we can conclude that with probability at least $1-\delta$, the constructed algorithm $\mathcal{A}_3$ solves $\mathcal{P}_2$ with error probability at most $\delta$.
	
	Next, we prove the sample complexity of $\mathcal{A}_3$. The calls of $\mathcal{A}_2$ return a series of arms, say $a_1,a_2,a_3,...$. Define a map $s$ such that $s(a_j)$ is the \textit{hard} set that $a_j$ belongs to. For $i\in[k]$, define $\tau_i:=\inf\{j:|\{s(a_1),s(a_2),...,s(a_j)\}|\geq i\}$, i.e., $\tau_i$ is the number of arms returned when $i$ \textit{hard} sets have been tagged. Also, let $\tau_0=0$.
	
	To calculate $\mathbb{E}\tau_i$, we observe that when there are $(i-1)$ tagged \textit{hard} sets, the probability that a new \textit{hard} set will be tagged after one more found arm is $1-\frac{i-1}{2k}$. Thus, by the property of geometric distributions, we have
	\begin{align}
	\mathbb{E}(\tau_i-\tau_{i-1}) = \frac{2k}{2k+1-k},
	\end{align}
	which implies
	\begin{align}
	\mathbb{E}\tau_k = \sum_{i=1}^k\mathbb{E}(\tau_i-\tau_{i-1})
	= \sum_{i=1}^k{\frac{2k}{2k+1-k}} \leq 2k.
	\end{align}
	
	Each call of $\mathcal{A}_2$ returns $k$ arms, and thus, after $O(1)$ expected number of calls of $\mathcal{A}_2$, $\mathcal{A}_3$ returns. Each call of $\mathcal{A}_2$ is with error probability $\frac{\delta}{2T}$ (recall $T=\lceil 2\log\frac{2k}{\delta}\rceil$). So by the definition of $\mathcal{A}_2$, each call conducts $o(\frac{k}{\epsilon^2}(\frac{n}{m}+\log\frac{2kT}{\delta})) = o(\frac{k}{\epsilon^2}(\frac{n}{m}+\log\frac{k}{\delta}))$ samples. This completes the proof of the sample complexity.
	
	The constructed algorithm $\mathcal{A}_3$ solves $\mathcal{P}_2$ with expected sample complexity $o(\frac{k}{\epsilon^2}(\frac{n}{m}+\log\frac{k}{\delta}))$. This completes the proof of Claim~\ref{UP-A3}.\end{proof}

If the $\mathcal{A}_2$ assumed in Claim~\ref{UP-A3} exists, it will lead to a contradiction to Claim~\ref{LB-P2}. This completes the proof of Theorem~\ref{LB-Q-FK}. $\square$.

\section{PROOF OF THEOREM~\ref{TP-AL-Q-IK}}
\textsl{Proof.}	Let $k\in\mathbb{Z}^+$, $\rho,\epsilon,\delta\in(0,\frac{1}{2})$, $\lambda\leq\lambda_\rho$ be given. For $p,x\in(0,1)$, we define $U_p:=\{a\in\mathcal{S}:\mu_a\geq\lambda_p\}$, $E_x:=\{a\in\mathcal{S}:\mu_a\geq\lambda-x\}$, and $F_x:=\mathcal{S}-E_x=\{a\in\mathcal{S}:\mu_a<\lambda-x\}$.	

By (\ref{ICDF1}), an arm randomly drawn from $\mathcal{S}$ is in $U_\rho$ with probability at least $p$. In the $t$-th repetition, by the choice of $n_1$ in AL-Q-IK, we have
\begin{align}
\mathbb{P}\{|A_t\cap U_\rho| = 0 \} & \leq (1-\rho)^{n_1} \nonumber \\ 
& = e^{n_1 \log(1-\rho)} \leq e^{-n_1 \rho} \leq \frac{1}{3}. \label{UrhoCovered} 
\end{align}
Given the condition $|A_t\cap U_\rho| > 0$, since $a_t$ is the returned value of Median-Elimination$(A_t, \epsilon_1, \frac{1}{4})$, by Theorem 4 \citep{MedianElimination2002}, $a_t$ is with probability at least $\frac{3}{4}$ in $E_{\epsilon_1}$. Thus, we can conclude that 
\begin{equation}\label{TopArmCovered}
\mathbb{P}\{a_t\in E_{\epsilon_1}\}\geq (1-\frac{1}{3})\frac{3}{4} = \frac{1}{2}.
\end{equation}
In Line~4, we sample $a_t$ for $n_2$ times, and its empirical mean is $\hat{\mu}_t$. Define $\mathcal{E}_t:=$ the event that $a_t$ is included in the returned value $Ans$. Since $\mathcal{E}_t$ happens if and only if $\hat{\mu}_t \geq \lambda - \epsilon_1-\epsilon_2$, by Hoeffding's Inequality and $n_2=\lceil\frac{1}{2\epsilon_2^2}\log\frac{k}{\delta}\rceil$, it holds that 
\begin{gather} 
\mathbb{P}\left\{\mathcal{E}_t^\complement \mid a_t\in E_{\epsilon_1} \right\} \leq \exp\left\{-2n_2\left(\epsilon_2^2\right)\right\}\leq \frac{\delta}{k},\label{PGoodAdded} \\
\mathbb{P}\left\{\mathcal{E}_t \mid a_t\in F_{\epsilon} \right\} \leq \exp\left\{-2n_2\left(\epsilon_2^2\right)\right\} \leq \frac{\delta}{k}.\label{PBadAdded}
\end{gather}
Since $\{a_t\in E_{\epsilon_1}\} \cap \{\hat{\mu}_t \geq \lambda - \epsilon_1-\epsilon_2\} \subset \mathcal{E}_t$, by (\ref{TopArmCovered}) and (\ref{PGoodAdded}), we have 
\begin{equation}\label{Padd}
\mathbb{P}\{\mathcal{E}_t\} \geq \frac{1}{2}(1-\frac{\delta}{k})\geq \frac{1}{4}. 
\end{equation} 
Besides, by (\ref{TopArmCovered}), (\ref{PGoodAdded}), and (\ref{PBadAdded}), we have
\begin{align} 
\frac{\mathbb{P}\left\{a_t\in E_{\epsilon} \mid \mathcal{E}_t \right\}}{\mathbb{P}\left\{a_t\in F_{\epsilon} \mid \mathcal{E}_t \right\}} 
\geq \frac{\mathbb{P}\left\{a_t\in E_{\epsilon_1} \mid \mathcal{E}_t \right\}}{\mathbb{P}\left\{a_t\in F_{\epsilon} \mid \mathcal{E}_t \right\}} \nonumber\\
= \frac{\mathbb{P}\left\{a_t\in E_{\epsilon_1}\right\}\mathbb{P}\left\{\mathcal{E}_t \mid a_t\in E_{\epsilon_1}\right\}}{\mathbb{P}\left\{a_t\in F_{\epsilon}\right\}\mathbb{P}\left\{\mathcal{E}_t \mid a_t\in F_{\epsilon}\right\}} \nonumber\\ 
\geq \frac{\frac{1}{2}\cdot(1-\frac{\delta}{k})}{\frac{1}{2}\cdot\frac{\delta}{k}} = \frac{k}{\delta}-1.
\end{align} 
Since $\mathbb{P}\{a_t\in E_{\epsilon} \mid \mathcal{E}_t\}+\mathbb{P}\{a_t\in F_{\epsilon} \mid \mathcal{E}_t\}=1$, we can conclude that 
\begin{equation} 
\mathbb{P}\left\{a_t\in E_{\epsilon} \mid \mathcal{E}_t \right\} \geq 1-\frac{\delta}{k}. 
\end{equation} 
This shows that when an arm $a_t$ is added to $Ans$, with probability at least $1-\frac{\delta}{k}$, $a_t$ is in $E_\epsilon$. Thus, we have 
\begin{equation} 
\mathbb{P}\{\forall a_t\in Ans , a_t\in E_{\epsilon}\} \geq 1-\delta. 
\end{equation}
Thus, the returned arms of AL-Q-IK all have expected rewards no less than $\lambda-\epsilon$ with probability at least $1-\delta$. This completes the proof of correctness.

It remains to derive the sample complexity. In each repetition, the algorithm calls Median-Elimination$(A_t, \epsilon_1, \frac{1}{4})$ for once, and samples $a_t$ for $n_2$ times. Each call of Median-Elimination takes at most $O(\frac{n_1}{\epsilon^2})=O(\frac{1}{\rho\epsilon^2})$ samples \citep{MedianElimination2002}, and $n_2=O(\frac{1}{\epsilon^2}\log\frac{k}{\delta})$. Thus, each repetition takes $O(\frac{1}{\epsilon^2}(\frac{1}{\rho}+\log\frac{k}{\delta}))$ samples. By (\ref{Padd}), in each repetition, with probability at least $\frac{1}{4}$, one arm is added to $Ans$, and the algorithm terminates after $k$ arms are added to $Ans$. Obviously, after at most $4k$ repetitions in expectation, the algorithm returns. Thus, the expected sample complexity is $O(\frac{k}{\epsilon^2}(\frac{1}{\rho}+\log\frac{k}{\delta}))$. This completes the proof. $\square$

\section{PROOF OF LEMMA~\ref{Correctness-PACMaxing}}
\textit{Proof.}
Let $a_r$ be the returned arm. For any arm $a$ in $\mathcal{S}$, define
\begin{align}
\mathcal{E}^N_a\!:=\!\{\exists t,N^t(a)=N,\mu_a\!<\!L^t(a) \lor \mu_a\!>\!U^t(a)\},
\end{align}
i.e., the event that when $N^t(a)=N$, $\mu_a$ is not within the interval $[L^t(a),U^t(a)]$. Define the bad event $\mathcal{E}_{out}:=\bigcup_{a,N}\mathcal{E}^N_a$. By (\ref{UpperConfidenceBound}) and (\ref{LowerConfidenceBound}), we have that 
\begin{align}
\mathbb{P}\left\{\mathcal{E}^N_a\right\} \leq 2\delta^N.
\end{align}
Thus, by $k_1\geq 2\sum_t{t^\gamma}$ and the union bound, we have that \begin{align}
\mathbb{P}\{\mathcal{E}_{out}\}\leq \sum_{a,N}{\mathbb{P}\left\{\mathcal{E}^N_a\right\}}\leq n\sum_{N=1}^\infty{2\delta^N}\leq \delta.
\end{align}
Since $budget$ is large enough, the algorithm returns at Line~15. Let $t_0$ be the index of the iteration when the algorithm returns. By the return condition of Line~15, we have that for all $a\neq a_r$, $U^{t_0}(a)\leq L^{t_0}(a_r)+\epsilon$. By the definition of $\mathcal{E}_{out}$, when it does not happen, for all arms $a$, $\mu_a\in [L^t(a),U^t(a)]$ for all $t$, implying that 
\begin{align}
\mu_a\leq U^{t_0}(a) \leq L^{t_0}(a_r)+\epsilon\leq \mu_{a_r}+\epsilon. 
\end{align}
Thus, the returned arm $a_r$ is $(\epsilon,1)$-optimal with probability at least $1-\delta$. 
$\square$

\section{PROOF OF LEMMA~\ref{SampleComplexityPACMaxing}}
\textsl{Proof.}
In the proof, we assume that $\mathcal{E}_{out}$ does not happen. This event is defined in the proof of Lemma~\ref{Correctness-PACMaxing}, and does not happen with probability at least $1-\delta$. 

Let $\tau$ be the number of samples taken till termination. Define the set $T:=\{n+2i:i\in\mathbb{N},n+2i<\tau\}$. $T$ is the set of $t$ such that $a^t$ and $b^t$ are computed. For each arm $a$, define $X_a:=\sum_{t\in T}{\mathds{1}_{b^t=a}}$, the number of times that $b^t$ is $a$. Define $\mu^*:=\max_{a\in A}\mu_a$, $\Delta'_a:=\mu^*-\mu_a$, and $\Delta_a:=\frac{1}{2}\max\{{\epsilon},\Delta'_a\}$. Now, we are going to bound $X_a$.

Let $a$ be an arbitrary arm in $A$. Assume that at some time $t\in T$, 
\begin{align}\label{NValue}
N^t(a)\!\geq\!\frac{1}{\Delta_a^2}\!\max\!\left\{\!\log\frac{k_1n}{\delta},(\gamma\!+\!\frac{\gamma}{e})\log\!\frac{(\gamma\!+\!\frac{\gamma}{e})}{\Delta_a^2}\!\right\},
\end{align}
and we will show that either $b^t$ does not equal to $a$ or the algorithm returns before the next sample. 

Let $x=\frac{\gamma}{\Delta_a^2}$ ($x>4$ as $\Delta_a\leq\frac{1}{2}$ and $\gamma>1$). Since $N^t(a)\geq (1+e^{-1})x \log((1+e^{-1})x)>4$, we have that
\begin{align}
\frac{N^t(a)}{\log{N^t(a)}}&\stackrel{(i)}{>} \frac{(1+e^{-1})x \log((1+e^{-1})x)}{\log((1+e^{-1})x)+\log\log((1+e^{-1})x)}\nonumber\\
&=\frac{(1+e^{-1})x}{1+\frac{\log\log((1+e^{-1})x)}{\log((1+e^{-1})x)}}\stackrel{(ii)}{\geq} x,
\end{align}
where (i) is because $\frac{y}{\log{y}}$ is increasing for $y\geq e$, and (ii) is because $\frac{\log{y}}{y}\leq \frac{1}{e}$. It implies that 
\begin{align}\label{NBound1}
\frac{1}{2}N^t(a)> \frac{\gamma}{2\Delta_a^2}\log{N^t(a)}.
\end{align}
Also, by (\ref{NValue}) we have that 
\begin{align}\label{NBound2}
\frac{1}{2}N^t(a)\geq \frac{1}{2\Delta_a^2}\log\frac{k_1n}{\delta}.
\end{align}
Thus, adding (\ref{NBound1}) and (\ref{NBound2}), we have that 
\begin{align} 
N^t(a)> \frac{1}{2\Delta_a^2}\log\frac{k_1n(N^t(a))^\gamma}{\delta}.
\end{align}
It follows that 
\begin{align}
\sqrt{\frac{1}{2N^t(a)}\log\frac{k_1n (N^t(a))^\gamma}{\delta}} < \Delta_a.
\end{align}
Recall that in the algorithm, for arm $a$, we define $U^t(a):=u(\hat{\mu}^t(a),N^t(a),\delta^{N^t(a)})$ and $L^t(a):=l(\hat{\mu}^t(a),N^t(a),\delta^{N^t(a)})$ as ((\ref{HodffdingUpper}) and (\ref{HoefdingLower})). By the choice of $\delta^{N^t(a)}=\frac{\delta}{k_1n (N^t(a))^\gamma}$ in PACMaxing, and the choice of confidence bounds, we have that
\begin{gather}
U^t(a)-\hat{\mu}^t(a)=\hat{\mu}^t(a)-L^t(a)< \Delta_a,\\
U^t(a)-L^t(a)< 2\Delta_a.\label{BoundedDifference}
\end{gather}

Now, for this $a$, we will show that either the algorithm returns before the next sample or $b^t\neq a$. 

First we consider the case where $\Delta_a=\frac{\epsilon}{2}$. In this case, we assume $b^t=a$, and show that the algorithm will return before the next sample. Recall that we assume $\mathcal{E}_{out}$ does not happen. This means for any $t$ and arm $b\in A$, $\mu_b$ is in $[L^t(b),U^t(b)]$. Since $b^t=a$ and $b^t:=\arg\max_{b\in A}U^t(b)$, for all arms $b\neq a$, $U^t(a)\geq U^t(b)$. By (\ref{BoundedDifference}), $L^t(a)\geq U^t(a)-\epsilon\geq U^t(b)-\epsilon$. This means that the algorithm returns arm $a$ before the next sample as we already have $B(t)\leq \epsilon$.

Next, we consider the case where $\Delta_a=\frac{\Delta'_a}{2}$. Let $a^*$ be the most rewarding arm of $A$ (i.e., the arm with the largest mean reward). As $\Delta'_{a^*}=0<\epsilon$, $a$ is not $a^*$. Since $\mathcal{E}_{out}$ does not happen, by the definition of $\mathcal{E}_{out}$ and (\ref{BoundedDifference}), we have that $U^t(a)<L^t(a)+\Delta'_a\leq \mu_a+\Delta'_a\leq \mu^*\leq U^t(a^*)$, implying $b^t\neq a$.

Thus, we can conclude that when $\mathcal{E}_{out}$ does not happen,
\begin{align}
X_a\!\leq\!1\!+\!\frac{1}{\Delta_a^2}\!\max\!\left\{\!\log\frac{k_1n}{\delta},(\gamma\!+\!\frac{\gamma}{e})\log\!\frac{(\gamma\!+\!\frac{\gamma}{e})}{\Delta_a^2}\!\right\}.
\end{align}
Except the first $n$ samples, there is one $b^t$ sampled out of every two consecutive samples. Thus, with probability at least $1-\delta$, the number of samples taken before termination is at most 
\begin{align}\label{PACMaxingUpperBound}
&n+2\sum_{a\in A}X_a \nonumber\\ 
\leq &3n\!+\!\sum_{a\in A}\!\frac{2}{\Delta_a^2}\!\max\!\left\{\!\log\!\frac{k_1n}{\delta},(\gamma\!+\!\frac{\gamma}{e})\log\!\frac{(\gamma\!+\!\frac{\gamma}{e})}{\Delta_a^2}\!\right\}.
\end{align}
The desired sample complexity follows.

Since $\Delta_a\leq \frac{\epsilon}{2}$, the $budget$ value stated in this lemma is no less than that in (\ref{PACMaxingUpperBound}). This completes the proof.
$\square$

\section{PROOF OF LEMMA~\ref{TPLE}}
\textsl{Proof.}	The first step is to prove that with probability at least $1-\frac{2\delta}{5}$, the $m$-th most rewarding arm of $A_1$ is in $M:=\{a\in\mathcal{S}:\lambda_\rho\leq\mu_a\leq\lambda_{\rho/2}\}$. Here, we recall that $m:=\lfloor 1 + \frac{3}{4}\rho n_3\rfloor$ as defined in LambdaEstimation. To do it, we need to introduce an inequality directly derived from Chernoff Bound. Let $X^1,X^2,...,X^t$ be $t$ independent Bernoulli random variables, and for all $i$, $\mathbb{E}X^i\geq p$. Define $S:=\sum_{i=1}^t{X^i}$. Let $B(t,p)$ denote a Binomial random variable with parameters $t$ and $p$. For any $b \leq tp$, we have $\mathbb{P}\{S \leq b\} \leq \mathbb{P}\{B(t,p)\leq b\}$, and thus, by Chernoff Bound, 
\begin{equation}\label{ChernoffBound}
\mathbb{P}\{S \leq b\} \leq \exp\left\{-\frac{t}{2p}\left(p-\frac{b}{t}\right)^2\right\}.
\end{equation}

In this paper, we use $a\sim\mathcal{S}$ to denote that $a$ is randomly drawn from $\mathcal{S}$. By (\ref{ICDF1}) and (\ref{ICDF2}), we have
\begin{gather}
\mathbb{P}_{a\sim\mathcal{S}}\{a\in S_1\}\leq \frac{\rho}{2},\label{PS1}\\
\mathbb{P}_{a\sim\mathcal{S}}\{a\in S_2\}\geq \rho.\label{PS2}
\end{gather}

By the works of \citet{BinomialKLBound1989}, we have that for $x>tp$, 
\begin{align}\label{KLBinomialTail}
\mathbb{P}\left\{B(t,p)\geq x\right\}\leq \exp\left\{-tD_{KL}\left(\frac{x}{t}\mid\mid p\right)\right\},
\end{align}
where $D_{KL}(p||q):=p\log\frac{p}{q}+(1-p)\log\frac{1-p}{1-q}$. 

Define two sets 
\begin{gather}
S_1:=\{a\in A_1:\mu_a>\lambda_{\rho/2}\},\\
S_2:=\{a\in A_1:\mu_a\geq\lambda_\rho\}.
\end{gather}

By Inequalities~(\ref{PS1}) and (\ref{KLBinomialTail}), we have
\begin{align} \label{S1Size}
& \mathbb{P}\left\{|S_1|\geq \frac{3}{4}\rho n_3\right\} \leq \exp\left\{-n_3D_{KL}\left(\frac{3}{4}\rho\mid\mid \frac{1}{2}\rho\right)\right\}\nonumber \\
& = \exp\left\{\!-n_3\!\left[\frac{3}{4}\rho\log\frac{3}{2}\!-\!\left(\!1\!-\!\frac{3}{4}\rho\right)\!\log\left(1\!+\!\frac{\frac{\rho}{4}}{1\!-\!\frac{3\rho}{4}}\right)\right]\right\}\nonumber \\
& \leq \exp\left\{-n_3\rho\left[\frac{3}{4}\log\frac{3}{2}-\frac{1}{4}\right]\right\} \leq \frac{\delta}{5}, 
\end{align}

Also, by (\ref{ChernoffBound}) and (\ref{PS2}), it holds that 
\begin{equation} \label{S2Size}
\mathbb{P}\left\{|S_2|\leq \frac{3}{4}\rho n_3\right\} \leq \exp\left\{-\frac{n_3}{2\rho}\left(\frac{1}{4}\rho\right)^2\right\} \leq \frac{\delta}{5}. 
\end{equation} 
The above two statement (\ref{S1Size}) and (\ref{S2Size}) implies that with probability at least $1-\frac{2\delta}{5}$, $|S_1|<\frac{3}{4}\rho n_3$ and $|S_2|>\frac{3}{4}\rho n_3$. Recalling that $m = \lfloor 1 + \frac{3}{4}\rho n_3 \rfloor$, the $m$-th most rewarding arm of $A_1$ is in $M$ with probability at least $1-\frac{2\delta}{5}$.

The second step is to prove that $\mu_{\hat{a}}$ is in $[\lambda_\rho-\epsilon_1, \lambda_{\rho/2}+\epsilon_2]$ with probability at least $1-\frac{4\delta}{5}$. The call of Halving$(A_1, m, \epsilon_1, \frac{\delta}{5})$ returns an $m$-sized set of arms $A_2$, and with probability at least $1-\frac{\delta}{5}$, every arm $a$ in it has $\mu_a\geq \lambda'_{[m]}-\epsilon_1$ \citep{Halving2010}, where $\lambda'_{[m]}$ is the mean reward of the $m$-th most rewarding arm of $A_1$. We note that with probability at least $1-\frac{2\delta}{5}$, the $m$-th most rewarding arm of $A_1$ is in $M$, implying $\lambda'_{[m]}\geq \lambda_{\rho}$. Thus, we have
\begin{align}\label{H1C}
\mathbb{P}\left\{A_2\subset E_{\epsilon_1}\middle| |S_1|<\frac{3}{4}\rho n_3 < |S_2|\right\}\geq 1-\frac{\delta}{5}
\end{align}
Besides, by (\ref{S1Size}) and $|A_2|=m\geq \frac{3}{4}\rho n_3$, at least one arm $a^w$ of $A_2$ is in $M$ given that $|S_1|<\frac{3}{4}\rho n_3$. The call of Halving$_2(A_3, 1, \epsilon_2, \frac{\delta}{5})$ returns an arm $\hat{a}$ of $A_2$ having $\mu_{\hat{a}}\leq \mu_{a^w}+\epsilon_2 \leq \lambda_{\rho/2}+\epsilon_2$ with probability at least $1-\frac{\delta}{5}$ \citep{Halving2010} given that $|S_1|<\frac{3}{4}\rho n_3$, i.e.,
\begin{align}\label{H2C}
\mathbb{P}\left\{\mu_{\hat{a}}\leq \lambda_{\rho/2}+\epsilon_2 \middle| |S_1|<\frac{3}{4}\rho n_3 \right\}\geq 1-\frac{\delta}{5}.
\end{align}
It follows from $\hat{a}\in A_2$, the definition of $E_{\epsilon_1}$, (\ref{S1Size}), (\ref{S2Size}), (\ref{H1C}), and (\ref{H2C}) that
\begin{align}\label{Step2}
\mathbb{P}\left\{\mu_{\hat{a}}\in \left[\lambda_\rho-\epsilon_1, \lambda_{\rho/2}+\epsilon_2\right]\right\} \geq 1-\frac{4\delta}{5}.
\end{align}

The third step is to prove that $\hat{\lambda}$ is in $[\lambda_\rho-\epsilon, \lambda_{\rho/2}]$ with probability at least $1-\delta$. Since $\hat{a}$ is sampled for $n_4$ times, by (\ref{Step2}) and Hoeffding's Inequality, we have
\begin{align}
&\mathbb{P}\left\{\hat{\lambda}\notin\left[\lambda_\rho-\epsilon, \lambda_{\frac{\rho}{2}}\right]\right\}\nonumber\\
= & \mathbb{P}\left\{\hat{\mu}\notin\left[\lambda_\rho-\epsilon_1-\epsilon_3, \lambda_{\frac{\rho}{2}}+\epsilon_2+\epsilon_3\right]\right\}\nonumber\\
\leq & \mathbb{P}\left\{\mu_{\hat{a}}\notin\left[\lambda_\rho-\epsilon_1, \lambda_{\frac{\rho}{2}}+\epsilon_2\right]\right\} + \mathbb{P}\left\{\left|\hat{\mu}-\mu_{\hat{a}}\right|\geq\epsilon_3\right\}\nonumber\\
\leq & \frac{4\delta}{5} + 2\exp\left\{-2n_4\epsilon_3^2\right\} \leq \frac{4\delta}{5}+\frac{\delta}{5}\leq \delta.
\end{align}
This completes the proof of correctness.

It remains to prove the sample complexity. Line~4 uses $O(\frac{n_3}{\epsilon^2}\log\frac{m}{\delta})=O(\frac{1}{\rho\epsilon^2}\log^2\frac{1}{\delta})$ samples \citep{Halving2010}, Line~5 uses $O(\frac{m}{\epsilon^2}\log\frac{1}{\delta})=O(\frac{1}{\epsilon^2}\log^2\frac{1}{\delta})$ samples, and Line~6 uses $n_4=O(\frac{1}{\epsilon^2}\log\frac{1}{\delta})$ samples. The desired results follows by summing these three upper bounds up. $\square$

\section{PROOF OF THEOREM~\ref{TP-AL-Q-FK}}

\textit{Proof.}
Each call of AL-Q-IK is wrong with probability at most $\frac{\delta}{k}$. The correctness follows. 

By Theorem~\ref{TP-AL-Q-IK}, the $t$-th repetition uses $O(\frac{1}{\epsilon^2}(\frac{n+1-t}{m+1-t}+\log\frac{k}{\delta}))$ samples in expectation. For all $x\in(0,1]$, we have $\frac{\log(1+x)}{x}\geq \log{2}$. It implies 
\begin{align}
\log\frac{m+2-t}{m+1-t} = & \log\left(1+\frac{1}{m+1-t}\right)\nonumber\\
\geq & \frac{\log{2}}{m+1-t},
\end{align}
and thus, 
\begin{align}
&\sum_{t=1}^k\left\{\frac{1}{\epsilon^2}\left(\frac{n+1-t}{m+1-t}+\log\frac{k}{\delta}\right)\right\}\nonumber\\ 
\leq & \frac{n}{\epsilon^2\log{2}}\sum_{t=1}^k\log\frac{m+2-t}{m+1-t} + \frac{k}{\epsilon^2}\log\frac{k}{\delta}\nonumber\\
\leq & \frac{n}{\epsilon^2\log{2}}\log\frac{m+1}{m+1-k} + \frac{k}{\epsilon^2}\log\frac{k}{\delta}.
\end{align}
The desired sample complexity follows.
$\square$

\section{ADDITIONAL NUMERICAL RESULTS}
First, we compare the pure exploration algorithms in the finite cases to demonstrate that by adopting the QE setting, the number of samples taken can be greatly reduced compared with the KE setting. 	
Other comparisons on the finite-armed algorithms are omitted as their performance is similar to their infinite-armed versions, especially when $n$ is large. Also, when $k=1$, their performance are almost the same. 

\begin{wrapfigure}{r}{0.23\textwidth}
	\includegraphics[scale=0.5]{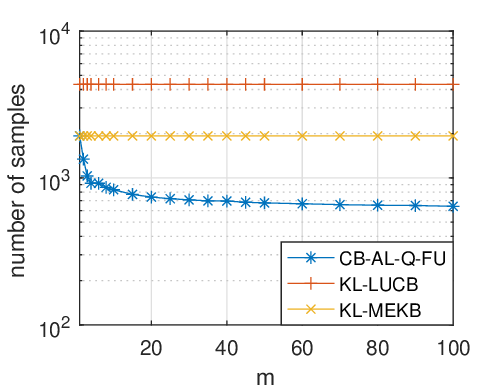}
	\caption{Comparison of the finite-armed pure exploration algorithms. $n=1000$, $k=1$, $\epsilon=0.05$, and $\delta=0.001$.}\label{fig:finite}
\end{wrapfigure}
The algorithms compared include CB-AL-Q-FK (CBB version of AL-Q-FK by replacing the subroutines with CBB ones, and in CB-AL-Q-IK, we choose $\epsilon_1=0.75\epsilon$), KL-LUCB for the finite case \citep{KLLUCB2013}, and MEKB \citep{FKLowerBound2004}. Here, we modify MEKB to the CBB version KL-MEKB. All the algorithms use the same confidence bounds given by \citet{KLLUCB2013} with $\gamma=2$. The results are summarized in Figure~\ref{fig:finite}. KL-LUCB and MEKB were designed to find one $(\epsilon, 1)$-optimal arm from a finite set. MEKB has the prior knowledge of $\lambda_{[1]}$, and can be regarded as the $m=1$ version of AL-Q-FK. There are totally 1000 arms. For each arm, its rewards follow the Bernoulli distribution, and its expected reward is generated by taking an independent instance of the Uniform([0,1]) distribution. All algorithms are tested on the same dataset. Every point is averaged over 100 independent trials.


Here, we note that the KE algorithms KL-LUCB and MEKB were designed to find an $(\epsilon,1)$-optimal arm, so their performance are independent of $m$.

According to Figure~\ref{fig:finite}, the two algorithms CB-AL-Q-FK and KL-MEKB that have knowledge of $\lambda_{[m]}$ or $\lambda_{[1]}$ perform better than KL-LUCB, the one without the knowledge, consistent with the theory. When $m=1$, the performance of CB-AL-Q-IK and KL-MEKB are close. However, when $m>1$, CB-AL-Q-IK takes less samples, and the gaps increases as $m$. The reason lies in that (CB-)AL-Q-IK's sample complexity depends on $\frac{n}{m}$ while (KL-)MEKB's depends on $n$. Thus, the numerical results indicate that by adopting the QE setting, one can find "good" enough arms by much less samples.

\begin{figure}[bht]\centering
	\begin{subfigure}[b]{0.23\textwidth}
		\includegraphics[scale=0.5]{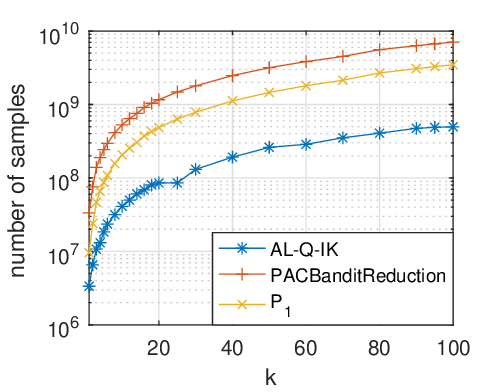}
		\caption{Vary $k$, $\rho=0.05$,$\epsilon=0.1$, and $\delta=0.01$.}
	\end{subfigure}
	\begin{subfigure}[b]{0.23\textwidth}
		\includegraphics[scale=0.5]{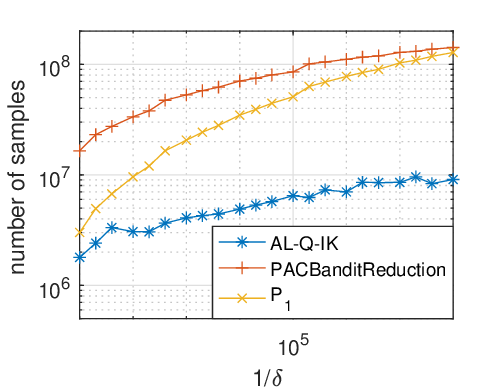}
		\caption{Vary $\delta$, $k=1$, $\rho=0.05$, and $\epsilon=0.1$.}
	\end{subfigure}\\
	\begin{subfigure}[b]{0.23\textwidth}
		\includegraphics[scale=0.5]{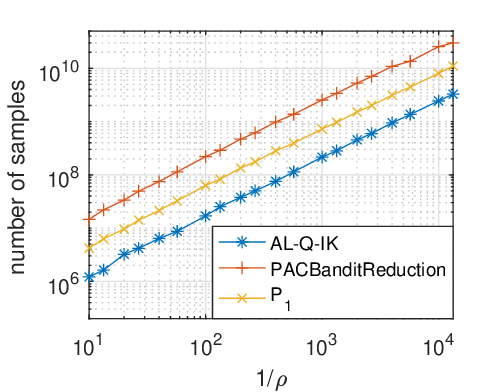}
		\caption{Vary $\rho$, $k=1$, $\epsilon=0.1$, and $\delta=0.01$.}
	\end{subfigure}
	\begin{subfigure}[b]{0.23\textwidth}
		\includegraphics[scale=0.5]{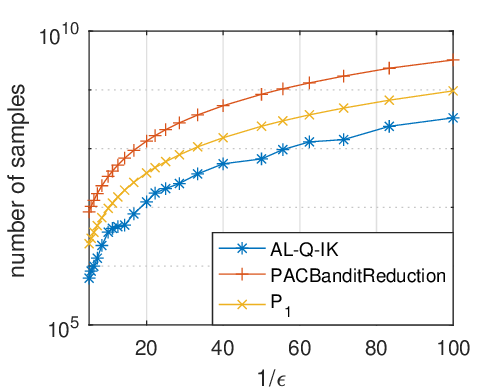}
		\caption{Vary $\epsilon$, $k=1$, $\rho=0.05$, and $\delta=0.01$.}
	\end{subfigure}
	\caption{Comparison of Non-CBB Algorithms.}\label{fig:Q-IKComparison}
\end{figure}

Next, we compare non-CBB algorithms: AL-Q-IK (Choosing $\epsilon_1=0.8\epsilon$), PACBanditReduction \citep{Q-IK2013}, and $\mathcal{P}_1$ \citep{Q-IU2017}. Here, again, we note that $\mathcal{P}_1$ does not require the knowledge of $\lambda_\rho$, but we want to illustrate how our algorithm along with this knowledge can improve the efficiency. The results are summarized in Figure~\ref{fig:Q-IKComparison} (a)-(d). In the simulations, the prior $\mathcal{F}$ is always Uniform([0,1]), and every point of every figure is averaged over 100 independent trials. 

The theoretical sample complexities of these three algorithms are: AL-Q-IK, $O(\frac{k}{\epsilon^2}(\frac{1}{\rho}+\log\frac{k}{\delta}))$; PACBanditReduction, $O(\frac{k}{\rho\epsilon^2}\log\frac{k}{\delta})$; $\mathcal{P}_1$, $O(\frac{k}{\rho\epsilon^2}\log^2\frac{k}{\delta})$. The numerical results confirm that AL-Q-IK performs better than the other two significantly. Figure~\ref{fig:Q-IKComparison}~(b) shows that AL-Q-IK's sample complexity increases slowly with $\frac{1}{\delta}$, consistent with the theory and numerical results on CB-AL-Q-IK.

According to Figure~\ref{fig:KLComparison} (c) and Figure~\ref{fig:Q-IKComparison} (c), the CB-AL-Q-IK's number of samples increases super-linearly with $\frac{1}{\rho}$ while that of AL-Q-IK increases linearly, consistent with the theory that the former depends on $\frac{1}{\rho}\log\frac{1}{\rho}$ while the latter depends on $\frac{1}{\rho}$. When $\frac{1}{\rho}$ is large enough, asymptotically AL-Q-IK will outperform CB-AL-Q-IK. However, in practice, under such small $\rho$ values, the sample complexity of both algorithms will be extremely large.

\end{document}